\documentclass[a4paper,12pt]{article}
\usepackage{epsfig,amsmath,amssymb,color,url}
\usepackage[latin1]{inputenc}

\usepackage[linesnumbered,ruled]{algorithm2e} 
\SetAlgoHangIndent{2em}

\usepackage{subfig}

\newtheorem{lemma}{Lemma}[section]
\newtheorem{definition}{Definition}[section]
\newenvironment{proof}{Proof:}{\hfill $\Box$}
\newtheorem{proposition}{Proposition}[section]

\newcommand{\entails}{\rightharpoonup}
\newcommand{\confidence}{\operatorname{conf}}
\newcommand{\support}{\operatorname{supp}}
\newcommand{\interest}{\operatorname{sign}}

\newcommand{\argmax}{\on{arg\,max}}

\renewcommand{\epsilon}{\varepsilon}

\newcommand{\with}{\, | \,}
\newcommand{\given}{\, | \,}
\newcommand{\sothat}{\, : \,}

\newcommand{\on}{\operatorname}

\newcommand{\Prob}{\mathbf{P}}
\newcommand{\cpi}{\boldsymbol{\pi}}
\newcommand{\crho}{\boldsymbol{\rho}}
\newcommand{\ct}{\boldsymbol{t}}

\newcommand{\bbS}{\overline{\mathbb{S}}_K}

\newcommand{\concat}{|}

\begin{document}
\title{\Large\bf MINING RANK DATA}

\author{Sascha Henzgen and Eyke H\"ullermeier\\
Heinx Nixdorf Institute\\
Paderborn University, Germany}

\maketitle

\begin{abstract}
The problem of frequent pattern mining has been studied quite extensively for various types of data, including sets, sequences, and graphs. Somewhat surprisingly, another important type of data, namely rank data, has received very little attention in data mining so far.  In this paper, we therefore addresses the problem of mining rank data, that is, data in the form of rankings (total orders) of an underlying set of items. More specifically, two types of patterns are considered, namely frequent rankings and dependencies between such rankings in the form of association rules. Algorithms for mining frequent rankings and  frequent closed rankings are proposed and tested experimentally, using both synthetic and real data. 
\end{abstract}


\section{Introduction}

The major goal of data mining methods is to find potentially interesting patterns in (typically very large) datasets. The meaning of ``interesting'' may depend on the application and the purpose a pattern is used for. Quite often, interestingness is connected to the \emph{frequency} of occurrence: A pattern is considered interesting if its number of occurrences in the data strongly deviates from what one would expect on average. When being observed much more often, ones speaks of a \emph{frequent pattern}, and the problem of discovering such patterns is called \emph{frequent pattern mining} \cite{han2007,fpm_2014}. The other extreme is outliers and exceptional patterns, which deviate from the norm and occur rarely in the data; finding such patterns is called \emph{exception mining} \cite{suzuki}. 

Needless to say, the type of patterns considered, of measures used to assess their interestingness, and of algorithms used to extract those patterns being highly rated in terms of these measures, strongly depends on the nature of the data. It makes a big difference, for example, whether the data is binary, categorical, or numerical, and whether a single observation is described in terms of a \emph{subset}, like in itemset mining \cite{han2007}, or as a \emph{sequence}, like in sequential pattern mining \cite{Mabroukeh2010}.

In this paper, we consider the problem of mining \emph{rank data}, that is, data that comes in the form of \emph{rankings} of an underlying set of items. This idea is largely motivated by the ubiquity of such data in many domains, such as information retrival, biology, psychology, and economics, as well as the recent emergence of \emph{preference learning} as a novel branch of machine learning \cite{mpub218}. While methods for problems such as ``learning to rank'' have been studied quite intensely in this field, rank data has hardly been considered from a data mining perspective so far. The motivation for mining rank data is manifold:
\begin{itemize}
\item Preferential patterns in the form of regularities, such as ``Chinese hotel guests mostly prefer room category \texttt{premier} to \texttt{studio} to \texttt{deluxe}'', and dependencies, such as ``Customers who prefer a mobile contract with flat rate to a contract without flat rate typically also prefer data volume 10 GB to 2 GB'' can be directly useful in many practical applications. Such applications include, for example, product recommendation, decision aiding, and constructive elicitation \cite{drag_cp17}, where an optimal combinatorial solution (e.g., a contract for a mobile phone) is constructed in a stepwise manner by fixing one decision (e.g., the rate) after the other (e.g., the data volume). In particular, returning a ranking of items or alternatives is often more useful than returning an unordered set, which does not prioritize between the candidates.
\item On a more methodological level, preferential patterns could be used to enhance or generalize other concepts for capturing dependence in statistics and data analysis. For example, just like standard association rules can be considered as a generalization of binary correlation (to directed relationships between multiple items), rules extracted from rank data, such as ``If $A$ precedes $B$ precedes $C$, then $D$ is likely to precede $E$, are directly connected to statistical measures of rank correlation. In contrast to the latter, they are not required to be globally valid (i.e., take all items into consideration). Instead, they may only capture local dependencies between subsets of items, which might be overlooked on a global scale.
\item Likewise, the mining of rank data has many potential applications in machine learning, especially in preference learning. For example, the approach of \emph{associative classification} \cite{thabtah} makes use of standard association rule mining for constructing multiclass classifiers. Roughly speaking, the idea is to find local patterns in the form of high-quality (frequent and confident) rules establishing associations between attribute values and class labels, and to combine a suitable selection of these rules into a global classifier. In very much the same way, preferential patterns in the form of rankings could be used to tackle the problem of \emph{label ranking}---an extension of multiclass classification, in which predictions in the form of rankings of all class labels are sought \cite{desa_ma11}.  
\end{itemize}

To illustrate what we understand by rank data in this paper, consider a version of the well-known SUSHI benchmark, in which 5000 customers rank 10 different types of sushi from most preferred to least preferred.\footnote{http://kamishima.new/sushi/} This data could be represented in the form of a matrix as follows:
\begin{verbatim}
            5  7  3  8  4 10  2  1  6  9
            6 10  1  4  8  7  2  3  5  9
            2  7  3  1  6  9  5  8  4 10
            .  .  .  .  .  .  .  .  .  .
\end{verbatim}
In this matrix, the value in row $i$ and column $j$ corresponds to the position of the $j^{th}$ sushi in the ranking of the $i^{th}$ customer. For example, the first customer likes the eighth sushi the most, the seventh sushi the second best, and so on. 

The above data consists of \emph{complete} rankings, i.e., each observation is a ranking of the complete set of items (the 10 types of sushi). There are many applications in which rank data is less complete, especially if the underlying set of items is large. For example, an incomplete version of the SUSHI data may look like this:
\begin{verbatim}
            2  ?  ?  3  ?  ?  ?  1  ?  ?
            ?  ?  1  3  ?  5  2  ?  4  ?
            1  4  ?  ?  ?  ?  3  5  2  ?
            .  .  .  .  .  .  .  .  .  .
\end{verbatim}
Here, the first customer provides a ranking of only three of the ten sushis, in which the eighth sushi is on position 1, the first on position 2 and the fourth on position 3, but without revealing any preferences on the remaining seven sushis; note that these positions are not considered as absolute (top-3) but relative, i.e., the eighth sushi is not necessarily the best among all ten. 

The problem of mining rank data was introduced in our previous work \cite{Henzgen2014}, albeit for the more restrictive case of complete rankings. Besides, a first algorithm for mining rank patterns in the form of what we call \emph{frequent subrankings} was introduced. The current paper considers the more general problem of mining incomplete rankings, presents an improved algorithm for mining frequent subrankings, as well as an algorithm for mining closed rankings. 



The paper is organized as follows. In the next section, we explain more formally what we mean by rank data and rank patterns, respectively. Then, following a review of related work in Section 3, we reconsider the problem of mining frequent rankings in Section \ref{sec:freq}, where we propose a new algorithm for this task that is more efficient than the one of \cite{Henzgen2014}. In Section \ref{sec:closed}, we introduce the notion of a \emph{closed} ranking and develop an algorithm for mining closed rankings.  Experiments with synthetic and real data, which are mainly meant to analyze the efficiency of our algorithms, are presented in Section \ref{sec:results}, prior to concluding the paper in Section \ref{sec:conclusion}.

\section{Rank data and rank patterns}

Let $\mathbb{O} = \{ o_1, \ldots , o_K \}$ be a set of items or objects. A ranking of these items is a total order that is represented by a permutation
$$
\cpi:\, [K] \longrightarrow [K] \enspace ,
$$
that is, a bijection on $[K] = \{ 1, \ldots , K \}$, where $\cpi(i)$ denotes the position of item $o_i$. Thus, the permutation $\cpi$ represents the order relation
$$
o_{\crho(1)} \succ o_{\crho(2)} \succ \cdots \succ o_{\crho(K)}
\enspace ,
$$
where $\crho = \cpi^{-1}$ is the inverse of $\cpi$, i.e., $\crho(j) = \cpi^{-1}(j)$ is the index of the item on position $j$. Often, $\cpi$ is called a ranking and $\crho$ an order. We denote the set of all rankings of $\mathbb{O}$ (permutations of $[K]$) by $\mathbb{S}_K$.   

A concrete ranking such as $\cpi = [3,1,4,2]$ will be written with brackets, and an order $\crho = (o_2 \succ o_4 \succ o_1 \succ o_3)$ or simply $\crho = (2,4,1,3)$ in parentheses. For the sake of clarity, objects $o_1, o_2, o_3$, etc.\ will often be denoted by letters $a, b, c$, etc.; for example, the previous order would then be $\crho = (b, d, a, c)$. Moreover, we shall often use the term ``ranking'' to refer to both rankings and orders---this should be uncritical due to the one-to-one correspondence between both concepts and, morever, the use of different symbls $\cpi$ and $\crho$. 

\subsection{Complete and incomplete rankings}

For a subset $O \subset \mathbb{O}$, we call a ranking $\pi$ of $O$ an incomplete ranking of $\mathbb{O}$ or a \emph{subranking}. For objects $o_i \in O$, $\pi(i)$ is the position of $o_i$ in the ranking, whereas $\pi(j) = 0$ for $o_j \notin O$. For instance, $\pi = [0,3,0,0,1,2]$ encodes a ranking of the set $\{o_2,o_5,o_6\}$, the corresponding incomplete order or  \emph{suborder} of which is given by $\rho = (5,6,2)$.
In the following, we will write complete rankings $\cpi$ and orders $\crho$ in bold font (as we already did above), whereas rankings $\pi$ and orders $\rho$ written in normal font are (possibly though not necessarily) incomplete. 

The number of items included in a subranking $\pi$ is denoted $|\pi|$; if  $|\pi|=k$, then we shall also speak of a $k$-ranking, and the same notation is used for orders. The set of all $k$-rankings is called $\mathbb{S}_{K,k}$, and the set of all complete and incomplete rankings will be referred to as 
$$
\overline{\mathbb{S}}_K = \bigcup_{k=2}^K \mathbb{S}_{K,k} \, .
$$
Again, for the sake of simplicity, we shall subsequently use the term ``ranking'' in a general way, so that it may refer to both complete and incomplete rankings.

\subsection{Restrictions and extensions}

We denote by $O(\pi)$ the set of items ranked by a subranking $\pi$, i.e., $O(\pi) = \{ i \in [K] \with \pi(i) > 0 \}$. The other way around, if $O' \subset O(\pi)$, then $(\pi|O')$ denotes the \emph{restriction} of the ranking $\pi$ to the set of objects $O'$, i.e., 
$$
(\pi|O')(j) = \left\{
\begin{array}{cl}
\#\{ o_i \in O' \with \pi(i) \leq \pi(j) \} &  \text{ if } o_j \in O' \\
0 &  \text{ if } o_j \notin O'
\end{array} \right. \, .
$$
The notations $O(\rho)$ and $(\rho | O')$ for orders will be used analogously. 

If $\pi$ is a subranking of $O = O(\pi)$, then $\cpi \in \mathbb{S}_K$ is a (linear) extension of $\pi$ if $(\cpi|O) = \pi$; in this case, the items in $O$ are put in the same order by $\cpi$ and $\pi$, i.e., the former is consistent with the latter. We shall symbolize this consistency by writing $\pi \subset \cpi$ and denote by $E(\pi) \subset \mathbb{S}_K$ the set of linear extensions of $\pi$. 

For (incomplete) rankings $\pi$ and $\pi'$, we define $\pi \subset \pi'$ if $E(\pi') \subset E(\pi)$; if $\pi \subset \pi'$, then $\pi$ is a restriction or subranking of $\pi'$, and $\pi'$ is an extension or superranking of $\pi$. Both relations are reflexive and include equality as a special case. In order to express that $\pi$ is a proper subranking of $\pi'$ and $\pi'$ a proper superranking of $\pi$, we shall write $\pi \subsetneq \pi'$; thus, $\pi \subsetneq \pi'$ iff $\pi \subset \pi'$ and $\pi' \not\subset \pi$.

\subsection{Frequent rankings}

Assume data to be given in the form of a set 
\begin{equation}\label{eq:data}
\mathbb{D} = \{ \pi_1 , \pi_2, \ldots , \pi_N \}
\end{equation}
of (possibly incomplete) rankings $\pi_i$ over the set of items $\mathbb{O}$. The tuple $(\mathbb{O}, \mathbb{D})$ plays the role of our ``database'', and each ranking $\pi_i$ is a ``transaction'' in this database. Returning to our example above, $\mathbb{O} = \{ o_1, \ldots , o_{10}\}$ could be the 10 types of sushi, and $\pi_i$ the ranking of some of these sushis by the $i$th customer.

Now, we are ready to define the notion of \emph{support} for a ranking $\pi \in \overline{\mathbb{S}}_K$. In analogy to the well-known problem of itemset mining, this is the relative frequency of observations in the data in which $\pi$ occurs as a subranking:
\begin{equation}\label{eq:supp}
\support(\pi) = \frac{1}{N} \cdot \#   \big\{ \pi_i \in \mathbb{D} \with \pi \subset \pi_i \big\} 
\end{equation}
A \emph{frequent subranking} is a subranking $\pi$ such that 
$$
\support(\pi) \geq \delta \enspace ,
$$
where $\delta$ is a user-defined support threshold. A frequent ranking $\pi$ is \emph{maximal} if all its (proper) superrankings $\pi'$ are non-frequent, i.e., 
$$
\pi \text{ maximal } \quad\equiv\quad  \support(\pi) \geq \delta   \wedge  \, \forall \pi' \sothat ( \pi \subsetneq \pi')  \Rightarrow  ( \support(\pi') < \delta)  \, .
 $$
 Moreover, a frequent ranking $\pi$ is \emph{closed} if all its superrankings $\pi'$ have a lower support, i.e.,
\begin{equation}\label{eq:closed}
\pi \text{ closed } \quad\equiv\quad  \support(\pi) \geq \delta   \wedge  \, \forall \pi' \sothat ( \pi \subsetneq \pi')  \Rightarrow  \big( \support(\pi') < \support(\pi)  \big)  \, .
\end{equation}
Obviously, the notions of support, maximality, and closedness, which have been introduced here for rankings, can be used for orders in exactly the same way. 

\subsection{Association rules}

Association rules are well-known in data mining and have first been considered in the context of \emph{itemset mining}. Here, an association rule is a pattern of the form $I \entails J$, where $I$ and $J$ are itemsets. The intended meaning of such a rule is that a transaction containing $I$ is likely to contain $J$, too. In market-basket analysis, where a transaction is a purchase and items are associated with products, the association $\{ \mathtt{paper}, \mathtt{envelopes} \} \entails \{ \mathtt{stamps} \}$ suggests that a purchase containing paper and envelopes is likely to contain stamps as well.

Rules of that kind can also be considered in the context of rank data. Here, we look at associations of the form 
\begin{equation}\label{eq:ar}
\pi_A \entails \pi_B \enspace ,
\end{equation}
where $\pi_A, \pi_B \in \overline{\mathbb{S}}_K$ are rankings i.e., subrankings of $\mathbb{O}$.

For example, the rule $[3,1,0,0,2] \, \entails \, [0,0,2,1,0]$, or equivalently $(b, e, a) \, \entails \, (d , c)$ in terms of corresponding orders,  suggests that if $b$ ranks higher than $e$, which in turn ranks higher than $a$, then $d$ tends to rank higher than $c$. Note that this rule does not make any claims about the order relation between items in the antecedent and the consequent part. For example, $d$ could rank lower but also higher than $b$. In general, the (complete) rankings $\cpi$ that are consistent with a rule (\ref{eq:ar}) is given by $E(\pi_A) \cap E(\pi_B)$.

\subsubsection{Quality measures}
In itemset mining, the confidence measure
$$
\confidence(I \entails J) = \frac{\support(I \cup J)}{\support(I)}
$$
that is commonly used to evaluate association rules $I \entails J$ can be seen as an estimation of the conditional probability 
$$
\Prob(J \given I) = \frac{\Prob(I \text{ and } J)}{\Prob(I)} \enspace ,
$$ 
i.e., the probability to observe itemset $J$ given the occurrence of itemset $I$. Correspondingly, we define the confidence of an association $\pi_A \entails \pi_B$ as
\begin{align}\label{eq:arconf}
\confidence(\pi_A \entails \pi_B) & = \frac{\#\{ \pi_i \in \mathbb{D} \with \pi_A , \pi_B \subset \pi_i \}}{ \#\{ \pi_i \in \mathbb{D} \with \pi_A  \subset \pi_i \}} = \frac{\support(\pi_A \oplus \pi_B)}{\support(\pi_A)} \, , 
\end{align}
where
\begin{align}\label{eq:union}
\pi_A \oplus \pi_B & 
= \Big\{ \pi \,\big\vert\,  O(\pi) = O(\pi_A) \cup O(\pi_B), \, (\pi | O(\pi_A))= \pi_A, \, (\pi | O(\pi_B))= \pi_B \Big\}\\
& = \big( E(\pi_A) \cap E(\pi_B) \,\big\vert\, O(\pi_A) \cup O(\pi_B)  \big) \nonumber
\end{align}
and
\begin{equation}
\support( \pi_A \oplus \pi_B)   = \# \big\{ \pi_i \in \mathbb{D} \with \, \exists \, \pi \in \pi_A \oplus \pi_B \sothat \pi \subset \pi_i \big\} 
\enspace .
\end{equation}
According to (\ref{eq:union}), $\pi_A \oplus \pi_B$ is the set of all consistent combinations of $\pi_A$ and $\pi_B$. Note that, if both $\pi_A \subset \pi$ and $\pi_B \subset \pi$, then at least one of these combinations must indeed occur in $\pi$; this is why the second equality in (\ref{eq:arconf}) holds. 

As an important difference between mining itemsets and mining rank data, note that the class of patterns is closed under  conjunction in the former but not in the latter case: Requiring the simultaneous occurrence of itemset $I$ \emph{and} itemset $J$ is equivalent to requiring the occurrence of their union $I \cup J$, which is again an itemset. As opposed to this, the conjunction (\ref{eq:union}) of two rankings $\pi_A$ and $\pi_B$ is not again a ranking, but a \emph{set} of rankings, namely all consistent combinations of $\pi_A$ and $\pi_B$.
As we shall see later on, this has an implication on an algorithmic level.

Finally, and again in analogy with itemset mining, we can define a measure of \emph{interest} or \emph{significance} of an association as follows:
\begin{equation}\label{eq:arint}
\interest(\pi_A \entails \pi_B) = \confidence(\pi_A \entails \pi_B) - \support(\pi_B)
\end{equation}
Just like for the measure of support, one is then interested in reaching  certain thresholds, i.e., in finding association rules $\pi_A \entails \pi_B$ that are highly supported, confident, and/or significant.

\subsubsection{Simplifying association rules} 

In itemset mining, the antecedent $I$ and consequent $J$ of an association rules $I \entails J$ must be disjoint $(I \cap J = \emptyset)$ to avoid trivial dependencies. In fact, assuming an item $a$ in the rule antecedent trivially implies its occurrence in all transactions to which this rule is applicable. In our case, this is not completely true, since a subranking is modeling relationships between items instead of properties of single items. For example, a rule such as $(a , b) \entails (a , c)$ is not at all trivial, although the item $a$ occurs on both sides. On the other hand, $(a,b,c) \entails (b,e)$ is equivalent to $(a,b,c) \entails (a,b,e)$ in the sense that both rules share the same set of positive examples, namely $(a,b,c,e)$ and $(a,b,e,c)$. Hence, the item $a$ can be removed from the consequent part without changing the meaning of the association rule. Or, stated differently, $a$ is a redundant consequence in the longer rule $(a,b,c) \entails (a,b,e)$. 
In contrast to this example, $a$ is not redundant in the rule $(a,b,c) \entails (a,e,b)$, because the shorter rule $(a,b,c) \entails (e,b)$ adds $(e,a,b,c)$ as a consistent ranking. 
Needless to say, our interest is to remove redundancy from association rules, making their consequent parts as simple and as short as possible, though without loosing equivalence to the original rule. In the following, we always assume an association rule $\pi_A \entails \pi_B$ to be consistent in the sense that $\pi_A \oplus \pi_B \neq \emptyset$, and non-trivial in the sense that $O(\pi_B) \not\subseteq O(\pi_A)$.


\begin{definition}
An item $i$ is called redundant in an association rule $\pi_A \entails \pi_B$ if $i \in O(\pi_B)$ and
$\pi_A \oplus \pi_B = \pi_A \oplus \pi_B'$, where $\pi_B' = \pi_B|(O(\pi_B)\setminus\{i\})$. An association rule $\pi_A \entails \pi_B$ is called redundant-free if not including any redundant item.
\label{def:redu}
\end{definition}

\begin{lemma}
An item $i \in O(\pi_B)$ is redundant in a rule $\pi_A \entails \pi_B$ if and only if the following holds: $i \in O(\pi_A)$, and in the consequent part $\pi_B$, $i$ is not directly adjacent to (i.e., immediately preceded or followed by) an item $j \in O(\pi_B) \setminus O(\pi_A)$.
\label{lem:conRedundancy}
\end{lemma}
\begin{proof}
First, note that removing an item from the consequent part of an association rule $\pi_A \entails \pi_B$ can only increase (but never decrease) the set of rankings $\pi_A \oplus \pi_B$ consistent with that rule. Moreover, it is obvious that an item $i \not\in O(\pi_A) \cap O(\pi_B)$, which only occurs in the consequent but not in the antecedent part, cannot be redundant (as its removal would even change the set of items). Thus, to remove redundancy, one can focus on items $i \in O(\pi_A) \cap O(\pi_B)$.

Consider such an item, and suppose the condition of the lemma to hold. Thus, we have a rule of the form $(\ldots, i , \ldots ) \entails (\ldots, a, i, b , \ldots )$ or $(\ldots, i , \ldots ) \entails (i, b , \ldots )$ or $(\ldots, i , \ldots ) \entails (\ldots, a, i)$, where both $a$ and $b$ also occur in the antecedent part. Then, the orders $a \succ i$ and $i \succ b$ are already implied by the rule antecedent, so that $i$ can be removed from the consequent part without adding any additional consistent ranking. 

Now, suppose the condition of the lemma does not hold. Thus, we have a rule $\pi_A \entails \pi_B$ of the form $(\ldots, i , \ldots ) \entails (\ldots, i, j , \ldots )$ or $(\ldots, i , \ldots ) \entails (\ldots, j, i, \ldots )$, where $j$ does not occur in the antecedent part. Removing $i$ as a consequent yields a rule $\pi_A \entails \pi_B'$ of the form $(\ldots, i , \ldots ) \entails (\ldots, j , \ldots )$. Consider the first case, where $\pi_B = (\ldots, i, j , \ldots )$. Since $j$ does not occur in $\pi_A$, and hence its position is only constrained by the rule consequent, there is a ranking $\pi \in \pi_A \oplus \pi_B$ in which $i$ and $j$ are directly neighbored, i.e., a ranking of the form $\pi = (\ldots , i,j, \ldots)$. The ranking $\pi'$ obtained from $\pi$ by swapping the positions of $i$ and $j$, i.e., $\pi' = (\ldots , j, i, \ldots)$, is consistent with $\pi_A \entails \pi_B'$ but not with $\pi_A \entails \pi_B$. Thus, $i$ cannot be removed from the consequent part. A similar argument applies to the second case, where $\pi_B = (\ldots,  j , i, \ldots )$.
\end{proof}

The above lemma suggests an easy way of simplifying association rules, so as to finally produce a non-redundant representation: As long as there are ``double-items'' $i$, i.e., items that occur both in the antecedent and consequent part, which are only neighbored by other double-items in the consequent, eliminate these items from the consequent. The only remaining question concerns the uniqueness of the rule eventually produced by a sequence of such simplifications. The answer to this question is affirmative, because the elimination of a redundant item does not influence the redundancy status of any other item: Since removing $i$ from $(\ldots, a, i, b , \ldots )$ yields $(\ldots, a, b , \ldots )$, only the status of $a$ and $b$ could be influenced. However, the right neighbor of $a$ is still a double-item, as before, and the left neighbor of $b$ is still a double-item, too. Thus, the order in which redundant items are eliminated does not matter, and the simplified rule eventually produced in unique.

\section{Related work}

The mining of rank data has connections to other frequent pattern mining problems, notably itemset mining and sequence mining. In this section, we provide a brief overview of approaches in these fields and position the mining of rank data as in-between itemset and sequence mining.

\subsection{Itemset mining}

\begin{table}
	\centering
	\caption{Database consisting of four itemsets.}
		\begin{tabular}{|c|c|c|}\hline
			itemset\_id & itemset & (ordered) frequent items\\\hline
			10 & $\{a,e,f,g\}$ &	$e$, $f$, $a$, $g$	\\\hline
			20 & $\{a,e,g,h \}$ 		& $e$, $a$, $g$, $h$			\\\hline
			30 & $\{c,d,e,f\}$ 		& $e$, $f$				\\\hline
			40 & $\{f,h\}$ 				& $f$, $h$  			\\\hline
		\end{tabular}
		\label{tab:iDB}
\end{table}

Given a set of objects/items $\mathbb{O} = \{o_1, o_2, \ldots ,o_K \}$, an itemset is a subset $I \subseteq \mathbb{O}$. An itemset database $\mathbb{D}$ is a collection of itemsets, also called transactions, each of which has a unique identifier (see Table \ref{tab:iDB} for an example). The support of an itemset $J$ is the number of itemsets $I \in \mathbb{D}$ such that $J \subset I$, and the itemset is frequent if its support exceeds a user-defined threshold. 

One of the earliest algorithms for mining frequent itemsets is Apriori introduced by Agrawal \cite{Agrawal1994} in 1994.  Apriori finds frequent itemsets in a level-wise manner, starting with singletons, then itemsets of size 2, size 3, and so forth. In each iteration, candidates of size $k+1$ are constructed from frequent itsemsets of size $k$; then, their frequency is counted by making one pass over the database. The efficiency of Apriori is mainly due to an effective pruning strategy: Since frequency is monotone decreasing on any chain of sets, each superset of a non-frequent itemset is necessarily non-frequent, too; or, stated differently, an itemset is disqualified as a candidate unless all its subsets have been confirmed to be frequent.

A major drawback of Apriori-like algorithms is the large number of database scans that are required to count frequencies. In \cite{Han2000}, Han et al.\ propose a technique called FP-growth. To avoid multiple scans, the database is first transformed into a compact representation called FP-tree, for which only two scans are needed. The first scan is used to find all frequent items, and to transform every transaction into an ordered list of frequent items. This list is used to build the FP-tree (see Figure \ref{fig:FPgrowth}). Except the root node, which is an empty node, every node contains a name (of an object) and a count. The actual mining process is now done on the FP-tree instead of the original database. Roughly, the idea is to build an item-conditional FP-tree for every item in the header table, and to repeat the procedure recursively. For example, the conditional pattern base of ``$g$'' is $\{(e\!:\!1,a\!:\!1),(e\!:\!1,f\!:\!1,a\!:\!1)\}$, and the conditional FP-tree only consist of the path $(e\!:\!2,a\!:\!2)$ with header table $(e, a)$ (Figure \ref{fig:FPgrowth}).

\begin{figure}
	\centering
		\includegraphics[width = 0.6\textwidth]{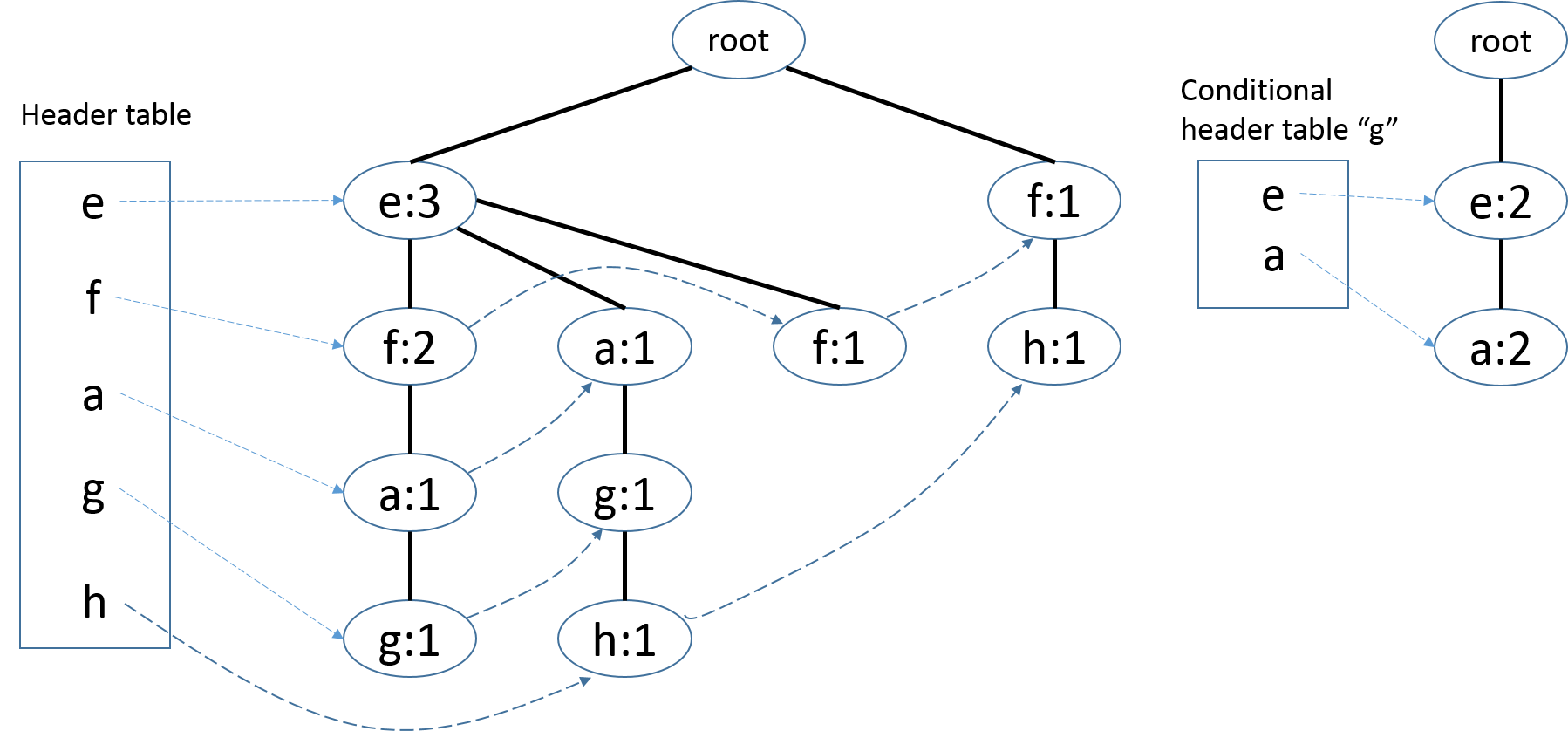}
		\caption{FP-tree (left) and "g"-conditional FP-tree (right)}
	\label{fig:FPgrowth}
\end{figure}

In the same year, Zaki \cite{Zaki2000} proposed different algorithms that are all based on two main ideas. The first is to transform the horizontal database, in which a set of items is given for every transaction, into a vertical database, in which a set of transactions is given for every item, namely those transactions in which the item is contained.  Let $L(J)$ be the set of all transactions containing itemset $J$. Using lattice theory \cite{Davey2002}, Zaki could show that $|L(J)| = \cap_{o_i \in J} L(o_i)$. Thus, the frequency of an itemset can be determined without scanning the entire database. The second idea is to divide the space of all itemsets, which form a Boolean lattice, into equivalence classes to reduce the amount of main memory needed during the mining process. Each equivalence class is again a Boolean lattice that can be mined separately.  The algorithms differ in the way they traverse the lattice. Eclat, for instance, uses a bottom-up search and recursive decomposition of the equivalence classes.

A-Close by Pasquier et al.\ \cite{Pasquier1999} is an algorithm to mine closed itemsets. It is based on the idea of mining generators and uses an Apriori-like strategy to produce them in a level-wise manner. A generator $G$ of a closed itemset $C$ is an itemset with $h(G) = C$, where $h: \, 2^{\mathbb{O}} \rightarrow 2^{\mathbb{O}}$ is the Galois connection.
First, all generators of length 1 are found. In the $i$-th iteration, all $(i-1)$-generators with the same $(i-2)$-prefix are combined to produce level $i$-generator candidates. These candidates are pruned in three steps:
(\emph{i}) All candidates with at least one $(i-1)$-subset that is not an $(i-1)$-generator are removed.
(\emph{ii}) If a candidate is not frequent, it is pruned.
(\emph{iii}) All candidates $G$ with an $(i-1)$-subset $J$ such that $h(G) = h(J)$ are pruned. 
When all generators are mined, all closed itemsets are generated by applying the Galois operation.

Pei et al.\ \cite{Pei2000a} extend FP-growth to mine closed itemsets and call the new algorithm CLOSET. Their approach is essentially based on two observations. First, if $C$ is a frequent closed itemset, then there is no item appearing in every transaction in the $C$-conditional database. Second, if an itemset $J$ is the set of items appearing in every transaction in the $C$-conditional database, then $C \cup J$ is a frequent closed itemset unless it is a subset of another closed itemset with the same support. In addition, CLOSET uses several optimization strategies to efficiently identify frequent closed itemsets and to reduce the search space.

CHARM by Zaki and Hsiao \cite{Zaki2002} exploits the Galois connection to mine closed itemsets. The algorithm operates on a prefix-tree representation of the itemset space $2^{\mathbb{O}}$. For this purpose, an order is defined on the items (for example, a lexicographic order). Every node in the tree is annotated with the set of transactions $t(J)$ containing the corresponding itemset $J$. CHARM starts the mining process with the smallest (in the sense of the specified order) item and tries to combine it with every other item. Let $I$ and $J$ be two itemsets to be combined, and $t(I) \cap t(J) \geq \delta$. Four cases are distinguished:
(\emph{i}) If $t(I) = t(J)$, every occurrence of $I$ can be replaced by $I \cup J$, and the $J$-branch can be pruned.
(\emph{ii}) If $t(I) \subset t(J)$, then $I$ can be replaced by $I \cup J$.
(\emph{iii}) If $t(I) \supset t(J)$, then $J$ can be replaced by $I \cup J$.
(\emph{iv}) If $t(I) \neq t(J)$, then $J \cup I$ is added to the tree, but neither $I$ nor $J$ can be pruned.
The process guarantees that all closed itemsets remain in the end.

\subsection{Sequence mining}

Let $\mathbb{O} = \{o_1, \ldots , o_K\}$ be a set of items. A sequence $s = \langle I_1, I_2, \ldots , I_n \rangle$ is an ordered set of itemsets $I_k \subseteq \mathbb{O}$.  A sequence database is a collection of such sequences, each of which has a unique identifier; see Table \ref{tab:FreeSpan} for an example. Given two sequences $s  = \langle I_1, I_2, \ldots , I_n \rangle$ and $t  = \langle J_1, J_2, \ldots , J_m \rangle$, $s$ is called a subsequence of $t$ if $n \leq m$ and there are indices $i_1 < i_2 < \ldots <i_n $ such that $I_k \subseteq J_{i_k}$ for every $k \in [n]$. A sequence $s$ is frequent if it is a subsequence of at least $\delta$ sequences in the database, where $\delta$ is a (user-defined) support threshold.
Additionally, given a sequence $s$ and an item $y$, we call $s \circ_i y = \langle(I_1)\ldots (I_n) \cup y\rangle$ an itemset-extension and $s \circ_s y = \langle(I_1)\ldots (I_n) (y)\rangle$ a sequence-extension. 

Agrawal and Srikant introduced the idea of sequence mining in 1995 \cite{Agrawal1995}. They present three mining algorithms: AprioriAll, AprioriSome, and DynamicSome. All three algorithms are adaptations of Apriori and exploit monotonicity for pruning, but also inherit the main disadvantage of this algorithm, namely the possibly large number of database scans that are  needed for counting the frequency of candidates.

One year later, Srikant and Agrawal published another Apriori-based algorithm called GSP (Generalized Sequential Patterns) \cite{Srikant1996}. GSP joins the set of frequent $(k-1)$-sequences $L_{k-1}$ with itself to generate candidates for the set $L_k$ of frequent $k$-sequences. More concretely, two sequences $s, t \in L_{k-1}$ are joined if they can be equalized by dropping one item of the first element in $s$ and one item of the last element in $t$. For every candidate thus produced, it is checked whether all contiguous $(k-1)$-subsequences are frequent (a contiguous subsequence of $s  = \langle I_1, I_2, \ldots , I_n \rangle$ is of the form $ \langle I_i, I_{i+1}, \ldots , I_{i+j} \rangle$ for some $i$ and $j$ such that $1 \leq i \leq i+j \leq n$). Yet another Apriori-based algorithm is PSP by Masseglia et al.\ \cite{masseglia1998, masseglia2000}. While being similar to GSP,  it makes use of a prefix-tree.

Zaki \cite{Zaki2001} introduces the algorithm SPADE, which can be seen as an adaptation of Eclat \cite{Zaki2000} to sequence mining. Like Eclat, SPADE represents the database in terms of a vertical transaction id-list and traverses a sequence lattice to find all frequent sequences. Given a $k$-sequence $s$, SPADE joins the id-lists of the two $(k-1)$-subsequences of $s$ with their shared $(k-2)$-prefix. If the cardinality of this join is greater or equal to the support threshold, certain lattice-theoretical properties imply that $s$ is frequent, too.  An algorithm similar to SPADE is SPAM by Ayres et al.\ \cite{Ayres2002}. Instead of regular and temporal joins, SPAM uses bitwise operations.

The algorithms so far can be seen as Apriori-like generete-and-test approaches. Another sort of algorithm tries to grow frequent patterns more directly without a candidate generation step. Early representatives of such pattern-growth algorithms are FreeSpan \cite{Han2000a} and PrefixSpan  \cite{Pei2001, Pei2004}, both using projected databeses, as well as WAP-mine \cite{Pei2000} and FS-Miner \cite{El-Sayed2003} based on tree-projection. In the following, we briefly outline the basic ideas of FreeSpan and PrefixSpan.

\begin{table}
	\centering
	\caption{Examplary sequence database. For simplicity, comma-separation between items is omitted, itemsets are written in brackets, and brackets are omitted for singletons.}
		\begin{tabular}{|c|c|}\hline
			Sequence\_id & Sequence 								\\\hline
			10 & $\langle a(abc)(ac)d(cf) \rangle$ 	\\\hline
			20 & $\langle (ad)c(bc)(ae)\rangle$ 		\\\hline
			30 & $\langle (ef)(ab)(df)cb\rangle$ 		\\\hline
			40 & $\langle eg(af)cbc\rangle$ 				\\\hline
		\end{tabular}
		\label{tab:FreeSpan}
\end{table}

Let $\mathbb{D}$ be the database given in Table \ref{tab:FreeSpan}. An element of $\mathbb{D}$ is a tuple $(id,s)$, where $s$ is a sequence and $id$ its unique identifier. FreeSpan tries to reduce the effort of database scans by dividing the database into smaller pieces by so-called database projections. In a first database scan, a list of all frequent items, called \textit{f\_list}, is build in decreasing order of frequency. For $\mathbb{D}$ in our example, $f\_list = [ a\!:\!4, b\!:\!4, c\!:\!4, d\!:\!3, e\!:\!3, f\!:\!3 ]$. Given $f\_list = [ x_1,x_2,\ldots,x_n ]$, the search space $\mathcal{S}$ is divided into $n$ disjoint subspaces $S_{\langle x_k\rangle} = \{ s \in \mathcal{S} \, | \, x_k \subset s, \forall x_l \in O(s): l < k\}$. For every subspace $\mathcal{S}_{\langle x_k\rangle}$, a level-2 projected database $\mathbb{D}_{\langle x_k\rangle} = \{ s \in \mathbb{D} \, | \, \langle x_k \rangle \subset s \}$ is created, where infrequent items and items succeeding $\langle x_k \rangle$ are ignored. Each level-1 projected database is scaned to find the length-2 frequent sequences contained in it. In our example, $\mathbb{D}_{\langle a\rangle} = \{ \langle aaa\rangle, \langle aa\rangle, \langle a\rangle, \langle a\rangle \}$ with $\langle aa\rangle : 2$ as the only length-2 sequence. To get length-3 frequent sequences, a level-2 projected database is build for every frequent lenght-2 sequence.  For example, $\mathbb{D}_{\langle aa\rangle} = \{ \langle aaa\rangle, \langle aa\rangle \}$.
Since there is no length-3 sequence with support threshold $\delta$, the mining process will stop in this branch.
If we follow the $\langle c\rangle$-branch, we first get $\mathbb{D}_{\langle c\rangle} = \{\langle a(abc)(ac)c\rangle,\langle ac(bc)a\rangle),\langle (ab)cb \rangle,\langle acbc\rangle \}$ and a scan results in the set $\{\langle ac\rangle \!:\!4,\langle cc\rangle\!:\!3, \langle bc\rangle\!:\!3,\langle cb\rangle\!:\!3,\langle (bc)\rangle\!:\!2,\langle ca\rangle\!:\!2\}$ of length-2 sequences. To get length-3 frequent sequences, a level-2 projected database is again  build for every frequent length-2 sequence.  In our example, $\mathbb{D}_{\langle ac\rangle} = \{ \langle a(abc)(ac)c\rangle,\langle ac(bc)a\rangle,\langle (ab)cb \rangle,\langle acbc\rangle \}$ with $\{\langle acb\rangle\!:\!3,\langle acc\rangle\!:\! 3,\langle (ab)c\rangle\!:\!2,\langle aca\rangle\!:\!2\}$ the set of frequent lenght-3 sequences. This procedure is repeated until the set of frequent sequences is empty for every branch.

PrefixSpan adopts the idea of FreeSpan but changes the way of database projection. The first difference is that PrefixSpan assumes an order of the items (e.g., a lexicographic order). Similarly to FreeSpan, a frequent item list $f\_list =[ x_1,x_2,\ldots,x_n ]$ is build first, and the search space $\mathcal{S}$ is divided into disjoint subspaces $S_{\langle x_k\rangle} = \{s \in \mathcal{S} \, | \, \langle x_k\rangle\ \text{is prefix of}\ s\}$. A sequence $t = \langle J_1 , J_2 ,  \ldots , J_m\rangle$ is \textit{prefix} of a sequence $s = \langle I_1 , I_2, \ldots , I_n\rangle$ if $m \leq n$, $J_k = I_k$ for $k\leq m-1$, $J_m \subseteq I_m$, and  $o_j < o_i$ for all $o_i \in I_m\setminus J_m$ and $o_j \in J_m$. The projected database corresponding  to $\mathcal{S}_{\langle x_k\rangle}$  is the set of all maximal subsequences in $\mathbb{D}$ whose prefix is $\langle x_k \rangle$. In addition, only the suffixes of the subsequences are considered. Given a sequence $s = \langle I_1 , \ldots , I_n \rangle$ and a prefix $\alpha = \langle J_1 ,\ldots , J_m \rangle$, a suffix is the sequence $\beta = \langle (I_m\setminus J_m) , I_{m+1} , \ldots ,  I_n \rangle$. In our example (Table \ref{tab:FreeSpan}), the projected database for $\langle a \rangle$ is $\mathbb{D}_{\langle a\rangle} = \{ \langle (abc)(ac)d(cf) \rangle, \langle (\_d)c(bc)(ae)\rangle, \langle (\_b)(df)cb\rangle, \langle(\_f)cbc \rangle \}$.
A scan of $\mathbb{D}_{\langle a\rangle}$ results in the frequent length-2 sequences with $\langle a\rangle$ as prefix: $\langle aa \rangle$, $\langle ab \rangle$, $\langle (ab) \rangle$, $\langle ac\rangle$, $\langle ad \rangle$, $\langle af \rangle$. These sequences are now used to divide the search space $\mathcal{S}_{\langle a\rangle}$ further into disjoint spaces $\mathcal{S}_{\langle aa\rangle}, \ldots , \mathcal{S}_{\langle af\rangle}$.
In our example, the database obtained for $\mathcal{S}_{\langle aa\rangle}$ is $\mathbb{D}_{\langle aa \rangle} = \{ \langle (\_bc)(ac)d(fc)\rangle, \langle (\_e)\rangle \}$.  It is easy to see that a scan of $\mathbb{D}_{\langle aa \rangle}$ results in no frequent length-3 sequence with $\langle aa\rangle$ as prefix, whence the algorithm stops in this branch. Finally, PrefixSpan terminates if no branch is processed any more.

Yan et al.\ \cite{Yan2003} propose CloSpan, an algorithm for mining closed sequences. A sequence $s$ is closed if it cannot be extended to a supersequence without decreasing support. Basically, CloSpan works like PrefixSpan but includes an early pruning step. Let $I(\mathbb{D}) = \sum_{s_i \in \mathbb{D}} l(s_i)$ be the size of $\mathbb{D}$, with $l(s_i)$ the number of items (sum of the cardinalities of the itemsets) in $s_i$. Yan et al.\ prove that given two sequences $s$ and $t$ such that $s \subset t$, the condition $\mathbb{D}_s = \mathbb{D}_{t}$ is equivalent to $I(\mathbb{D}_s) = I(\mathbb{D}_{t})$. In other words, if a sequence $s$ has a subsequence or supersequence $t$ such that $I(\mathbb{D}_s) =I(\mathbb{D}_t)$, and $t$ has already been explored, then there is no need to additionally explore $s$ since $\mathbb{D}_s = \mathbb{D}_t$. As the output of CloSpan may still contain non-closed sequences, a post-processing step is needed in which all non-closed sequences are eliminated.

COBRA by Huang et al.\ \cite{Huang2006} uses several pruning techniques while traversing the prefix-tree. The authors show that a closed sequence has only closed itemsets as elements.  Exploiting this observation, a sequence is extended with a closed itemset instead of a locally frequent item, thereby reducing the depth of the prefix-tree. In a second pruning step, called LayerPruning, a branch $\langle s \circ_s p_2 \rangle$, with $p_2$ a closed itemset, is pruned if there exists a supersequence $\langle s \circ_s p_1 \circ_s p_2 \rangle$ with $\mathbb{D}_{\langle s \circ_s p_2 \rangle} = \mathbb{D}_{\langle s \circ_s p_1 \circ_s p_2 \rangle}$.
Since these two steps are not sufficient to eliminate all non-closed sequences, a third pruning step called ExpPruning is applied: given two sequences $s$ and $t$, the former is removed if $s \subset t$ and both have the same support.

BIDE by Wang et al.\ \cite{Wang2007} also traverses the prefix-tree in a depth-first manner and prunes a branch if an item $o_i$ can be found such that, for the corresponding prefix $s = \langle I_1 , I_2, \ldots , I_n\rangle$, the sequence $\langle I_1, \ldots  , I_k \circ_i o_i,  I_{k+1} , \ldots  , I_n\rangle$ or $\langle I_1, \ldots  , I_k \circ_s o_i,  I_{k+1} , \ldots  , I_n\rangle$ is a subsequence in every sequence in $\mathbb{D}_{\langle s \rangle}$ for some fixed $k$.

A relatively new algorithm called ClaSP \cite{Gomariz2013} combines the vertical database representation of SPADE and the pruning technique of CloSpan. Experiments show that ClaSP outperforms both algorithms, SPADE and CloSpan.

One of the newest approaches is CSpan by Raju and Varma \cite{Raju2015}. It is based on PrefixSpan, too, and uses a technique for early detection of closed sequential patterns. For every projected database $\mathbb{D}_{\langle s\rangle}$, CSpan checks if there exists a frequent item $y$ that occurs in every sequence in $\mathbb{D}_{\langle s \rangle}$ as the same extension of $s$, thus either an itemset-extension (i.e., $s \circ_i y = \langle(I_1)\ldots (I_n) \cup y\rangle$) or a sequence-extension (i.e., $s \circ_s y = \langle(I_1)\ldots (I_n) (y)\rangle$). If no such item $y$ exists, $s$ must be closed and is added to the set of closed sequences. Independently of checking the occurrence, the algorithm recursively continues with the projected databases $\mathbb{D}_{\langle s \circ_i p\rangle}$ and $\mathbb{D}_{\langle s \circ_s p\rangle}$ for every frequent item $p$ in $\mathbb{D}_{\langle s\rangle}$. Although CSpan is not using a technique for pruning branches like CloSpan, it shows a better performance than CloSpan and ClaSP.

For a more detailed overview and taxonomy covering many sequence mining algorithms, the interested reader is referred to Mabroukeh and Ezeife \cite{Mabroukeh2010}.

\subsection{Itemset and sequence mining versus mining rank data}
\label{sec:IandSvsR}

The connection between mining rank data and itemset mining has already been touched upon several times. 
Indeed, noting that a ranking can be represented (in a unique way) in terms of a \emph{set} of pairwise preferences, one may wonder whether the former cannot simply be reduced to the latter. To this end, a new item $o_{i,j}$ is introduced for each pair of items $o_i , o_j \in \mathbb{O}$, and a subranking $\pi$ is represented by the set of items
$$
\{ o_{i,j} \with o_i , o_j \in O(\pi), \, \pi(i) < \pi(j) \} \enspace .
$$ 
This reduction has a number of disadvantages, however. First, the number of items is increased by a quadratic factor, although the information contained in these items is largely redundant. In fact, due to the transitivity of rankings, the newly created items exhibit (logical) dependencies that need to be taken care of by any mining algorithm. For example, not every itemset corresponds to a valid ranking, only those that are transitively closed. This already follows from the observation that the number of potential itemsets, which is given by $2^{K(K-1)}$, is much larger than the number of rankings, which is
\begin{equation}\label{eq:tnor}
\sum_{k=2}^K  {K \choose k }  k!  = \sum_{k=2}^K \frac{K!}{(K-k)!}  = K!  \sum_{k=0}^{K-2} \frac{1}{k!} \, .
\end{equation}
Correspondingly, algorithms like Apriori would need $\mathcal{O}(k^2)$ iterations to find a frequent ranking of length $k$.

Apart from that, there are some important differences between the two settings, for example regarding the number of possible patterns. In itemset mining, there are $2^K$ different subsets of $K$ items, which is much smaller than the $K!$ number of rankings of these items. However, the $K$ we assume for rank data (at least if the rankings are supposed to be complete) is much smaller than the $K$ in itemset mining, which is typically very large. Besides, the itemsets observed in a transaction database are normally quite small and contain only a tiny fraction of all items. In fact, assuming an upper bound $b$ on the size of an itemset, the number of itemsets is of the order $O(K^b)$ and grows much slower in $K$ than exponential.

The fact that a reduction of mining rank data to the problem of itemset mining can only be done at the cost of ``blowing up'' the representation is due the low expressivity of itemsets, whence the order of the elements in a ranking needs to be encoded in a large set of (meta-)items. The framework of sequence mining, on the other hand, is not less but even more expressive than our framework for mining rank data. Indeed, a ranking can be seen as a specific type of sequence, namely a sequence with only singleton-itemsets and without multiple occurrences of items. Thus, one may wonder whether rank data could not simply be mined using standard tools for sequence mining.  

One obvious reason not to do so is efficiency. It is clear that, being designed for solving more general problems, sequence mining algorithms do not exploit the specific properties of rank data. Therefore, when being applied to instances of this subclass of problems, they are likely to be less efficient than specifialized approaches.  
For example, as will become clearer in the next two sections, our algorithms for mining rank data heavily exploit the properties of irreflexivity, asymmetry, and transitivity, which are key characteristics of rankings. While these properties are implicit in both the representations we use and the algorithms operating on these representations, they need to be enforced in an artificial way in the case of itemset mining (by representing a ranking in terms of all pairwise preferences). Likewise, such properties are not used in sequence mining, simply because they do not hold; for example, the sequence $\langle a , b, a \rangle$ contains both $\langle a , b \rangle$  and $\langle b, a \rangle$ as subsequences, and both might be frequent in $\mathbb{D}$.  

Another reason is similar to the one already mentioned in the case of itemset mining and concerns properties of the data commonly encountered in practical applications. In fact, just like itemset mining, sequence mining typically assumes a setting with many items and sparse itemsets.  As an example, consider the FreeSpan algorithm, on which many other sequence miners are built. The technique of database projections used by this algorithm is indeed effective under the assumption of sequences of sparse itemsets but would hardly make sense for rankings in which most of the items occur most of the time.

\subsection{Other techniques}

Another pattern mining technique is \emph{ranked tiling} \cite{LeVan2014}. Given a database with rank data, the goal is to find large areas with high ranks, where an area or tile is a kind of bicluster identified by a subset of rows $R \subseteq \mathcal{R}$ and columns $C \subseteq \mathcal{C}$. Le Van et al.\ \cite{LeVan2014} only consider complete or tied rankings, but not incomplete ones.  
The first problem addressed in the paper is called \textit{maximal ranked tile mining}:
Find an as large as possible tile with a sufficiently high average rank:
\begin{equation}
(R^*,C^*) = \argmax_{R,C}\sum_{r\in R, c\in C}(\mathcal{M}_{r,c} - \theta) \, ,
\end{equation}
where $\mathcal{M}$ the rank data represented in the form of a matrix and $\theta$ is a threshold. 
The second problem, called \textit{ranked tiling}, is an extension, in which a collection of possibly more than one tile is sought. The goal, then, is to find a compromise between the size of the tiles (which should again be large) and their overlap (which should be small).  
The authors use constraint programming techniques and greedy heuristics to (approximately) solve both problems.
Although their approach is dealing with rank data, just like ours, the concept of a tile is obviously different from the concept of a ranking. In other words, while both approaches share the type of data, they differ with regard to the type of pattern mined in the data, and hence the type of algorithmic techniques used.

Rankings are also used by Calders et al.\ \cite{Calders2006}, although  more indirectly and again for a different purpose. They are interested in pattern mining on numerical data and make use of rank correlation measures in order to determine correlations between attributes in the database: For each attribute, the transactions are ranked according to their value on that attribute, and pairs of attributes are then compared in terms of the corresponding rankings. Thus, the rankings considered here are ``vertical'' (over the transactions), in contrast to our approach, where they are ``horizontal'' (over attributes/objects). Moreover, Calders et al.\ are not interested in discovering patterns in the form of rankings.

\section{Mining frequent rank patterns}
\label{sec:freq}

The algorithm for mining frequent rankings proposed by Henzgen and H\"ullermeier \cite{Henzgen2014} was mainly motivated by the well-known Apriori algorithm for mining frequent itemsets \cite{Agrawal1994}. It exploits a monotonicity property similar to the one that holds for itemsets, namely that the support of a ranking is upper-bounded by the support of all of its subrankings:
\begin{equation}\label{eq:mono}
(\pi \subset \pi') \quad \Rightarrow \quad \big( \support(\pi') \leq \support(\pi) \big) \enspace .
\end{equation}
This property allows for mimicking the basic Apriori procedure: Starting with frequent rankings of length 2, one iteratively constructs candidate rankings of length $k+1$ from frequent subrankings of length $k$ (pruning those that contain a non-frequent subranking and hence cannot be frequent themselves), and computes the support of these candidates by making one pass over the data. This latter step, computing the support of all candidate rankings, turns out to be very time-consuming in general.

In the following, we therefore introduce a new algorithm for mining frequent rankings, called \textit{Tail Extension Subranking Mining Algorithm} (TESMA). This algorithm will improve running time by several orders of magnitude at the expense of a small increase in space complexity. The basic idea is to span a prefix tree and traverse this tree in a depth-first manner. In addition, we are using a vertical database to avoid repeated database scans. Thus, for every ranking, we (implicitly) store the set of all transactions in which this ranking is contained. Formally, given a ranking $\pi$, we define the \emph{g-closure} of $\pi$ as follows:
\begin{equation}\label{eq:gclosure}
g(\pi) = \big\{  \pi_i \in \mathbb{D}  \with \pi \subset \pi_i  \big\}
\end{equation}
Obviously, given two rankings $\pi_A$ and $\pi_B$,
$$
\support( \pi_A \oplus \pi_B) = \# \big(  g(\pi_A) \cap g(\pi_B)  \big) \, .
$$
Recall that, according to (\ref{eq:union}), $\pi_A \oplus \pi_B$ is a set of rankings. Therefore, $\pi_i \in \mathbb{D}$ is an element of $g(\pi_A) \cap g(\pi_B)$ if it contains \emph{any} element of this set. As a consequence, even if $\pi_A \oplus \pi_B$ is frequent, this does not imply that the single elements of $\pi_A \oplus \pi_B$ are frequent, too.

However, in the special case where $\pi_A \oplus \pi_B$ consists of a single ranking, i.e.,$ \pi_A \oplus \pi_B = \{ \pi \}$, we obviously have 
$$
\support(\pi) = \support( \pi_A \oplus \pi_B) = \# \big(  g(\pi_A) \cap g(\pi_B)  \big)  \, .
$$
This is an important observation, which allows for the successive construction of (potentially) frequent subrankings without the need to pass through the whole data every time for counting the frequency. 

More specifically, our algorithm makes use of a combination that we call \emph{tail extension}. Consider a subranking $\pi$ with associated order $\rho = \pi^{-1}$, and let $last(\pi) = last(\rho) = \rho(|O(\pi)|)$ be the object ranked on the last position. A tail extension of $\rho$ is then given by $\rho' = \rho \oplus (last(\pi) \succ o)$, where $o \in \mathbb{O} \setminus O(\pi)$.\footnote{More correctly, since $\oplus$ yields a set of rankings, we should say that $\rho'$ is the unique element in $\rho \oplus (last(\rho) \succ o)$.}  As a shorthand notation, we write $\rho \concat o$ for the extension of $\rho$ by $o$.

\subsection{Representing transaction sets}

The basic element of TESMA is the storage and comparison of sets of transactions (g-closures).
To guarantee a fast implementation of operations like intersection, union, and equality, we represent a subset of transactions $T \subseteq \mathbb{D}$ as a bit vector $\ct$ of length $N= |\mathbb{D}|$. An entry $t_i$ of this vector is set to $1$ if $\pi_i \in T$, otherwise $t_i = 0$. Then, the intersection of sets of transactions can be computed by a bitwise AND, the union by a bitwise OR, and the equality by an XOR followed by testing whether the resulting vector is equal to the zero-vector.

The principle of the algorithm is quite simple. We implicitly represent the space of patterns (search space) as a set of prefix trees (Figure \ref{fig:prefixTree}) and traverse these trees in a depth-first manner. Every node, except the root nodes, represents an object, and every path, starting with a root node, an order. The algorithm starts by finding all frequent 2-rankings and the corresponding set of transactions (Algorithm \ref{algo:HHM}).  For every frequent 2-ranking, which builds the root of a tree, a depth-first search is used to construct longer rankings (Algorithm \ref{algo:DeepTESMA}). A ranking $\rho$ is extended by one of its children $o$, if $(o_l,o)$ is frequent with $o_l = last(\rho)$."

\begin{figure}
	\centering
		\includegraphics[width = 0.7\textwidth]{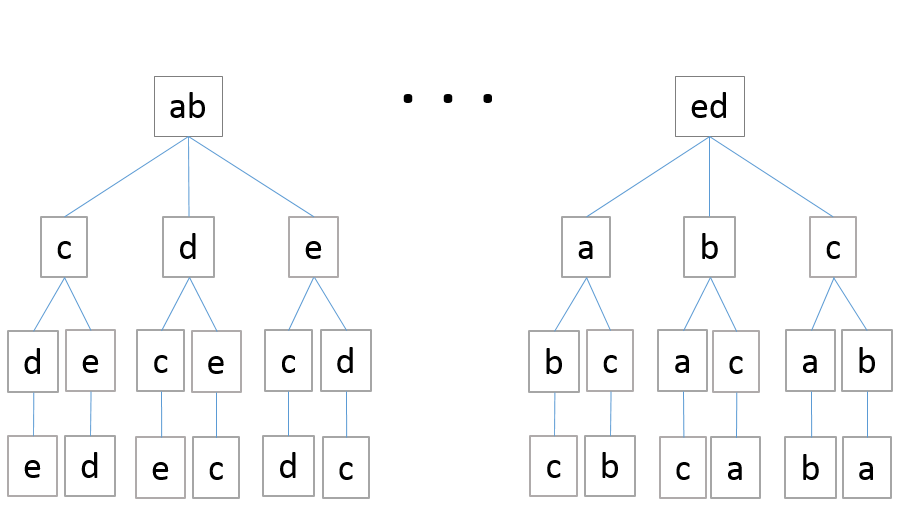}
		\caption{Search space represented as a set of prefix trees.}
		\label{fig:prefixTree}
\end{figure}

Let $o$ be one of these children. At this point, we already know the g-closure $g(\rho)$, which is stored in the bitvector $\ct_{\rho}$. We also know $g(o_l \concat o)$, which is stored in $\ct_{o_l \concat o}$. Therefore, we can easily compute $g(\rho \concat o)$ by an AND operation:
$$
\ct_{\rho  \concat o} = \ct_{\rho} \text{ AND } \ct_{o_l \concat o} 
$$
If $\#g(\rho  \concat o) \geq \delta$, then $\rho  \concat o$ is a frequent ranking that will be stored. The process then continues with one of the children of $o$. 
If $\#g(\rho   \concat o) < \delta$, we try to extend $\rho$ by another child of $o_l$.
The advantage of depth-first search is that, in addition to the transaction bitvectors of all 2-rankings, only the bitvectors of the rankings along the current path need to be stored, of which there are at most $N$.
For comparison, in breath-first search, the bitvectors of all nodes on the frontier would be stored, which are as many as $\mathcal{O}(N!)$.

\subsection{Storing 2-rankings and g-closure}

In both algorithms, TESMA and GPMiner, repeated access to frequent 2-rankings and their corresponding g-closures is required. Thus, the initial step of TESMA is to build a hash map that enables a quick access to the g-closures of frequent 2-rankings.


\begin{algorithm}
\SetAlgoNoLine
		$OOT \leftarrow \emptyset$\tcp*[r]{Outer hash-map (Figure \ref{fig:2rankHash})}
		\For{$i=1 \to |O|$}{
			$OT_i \leftarrow \emptyset$\tcp*[r]{Inner hash-map}
			$OOT.put(i,OT_i)$
			\For{$j=1 \to |O|$}{
				\If{$isFrequent((i,j))$}{
					$OT_i.put(j,g((i,j)))$\;
				}
			}
		}
		return $OOT$\;
	\caption{HHM}\label{algo:HHM}
\end{algorithm}

\begin{table}
	\centering
	\caption{Example of a hash map constructed for a set of transactions (rankings) and an absolute frequency of $2$. The Bit vectors have to be read from right to left.}
		\begin{tabular}{|c|c|}\hline
		\multicolumn{2}{|c|}{$\rho_1 = (a, b, e, c, d)$  $\rho_2 = (a, d, b, c, e)$}\\
		\multicolumn{2}{|c|}{$\rho_3 = (c, a, b, e, d)$  $\rho_4 = (b, a, d, c, e)$}\\\hline
		$a$ & \multicolumn{1}{c|}{$(b, [0111])$\ $(c,[1011])$\ $(d,[1111])$\ $(e,[1111])$}\\\hline
		$b$ & \multicolumn{1}{c|}{$(a,[1000])$\ $(c,[1011])$\ $(d,[1101])$\ $(e,[1111])$}\\\hline
		$c$ & \multicolumn{1}{c|}{$(d,[0101])$\ $(e,[1110])$}\\\hline
		$d$ & \multicolumn{1}{c|}{$(e,[1010])$}\\\hline
		$e$ & \multicolumn{1}{c|}{$(d,[0101])$}\\\hline
		\end{tabular}	
	\label{tab:exDataSet}
\end{table}

The algorithm HHM builds a hash map $OOT$, in which a \texttt{(key,value)} entry is stored for every object $o_i \in \mathbb{O}$, with $o_i$ as key and another hash map $OT_i$ as value. If an object pair $(o_i,o_j)$ is found to be frequent, the entry ($o_j,g(o_j)$) is added to the hash map $OT_i$ (Figure \ref{fig:2rankHash}). Thus, an access time of $\mathcal{O}(1)$ is achieved. An example is given in Table \ref{tab:exDataSet}.

\begin{figure}
	\centering
		\includegraphics[width = 0.5\textwidth]{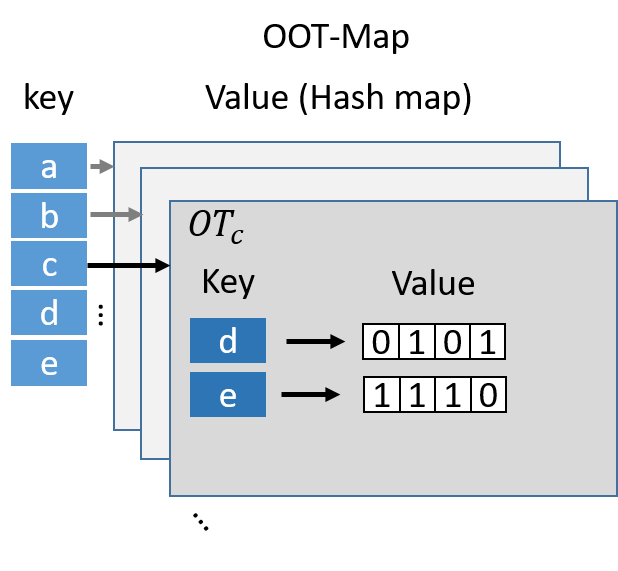}
		\caption{OOT is a hash map with objects $o_i$ as keys and (inner) hash maps as values. The keys of the inner hash map are again objects $o_j$, and the value is the vector of transactions in which the corresponding  2-ranking $(o_i,o_j)$ occurs.}
		\label{fig:2rankHash}
\end{figure}

\subsection{TESMA}

After the hash map $OOT$ is build, all frequent 2-rankings of the $OOT$ hash map are appended to the list of frequent patterns, serving as seeds for the recursion (Algorithm \ref{algo:TESMA}). Every frequent 2-ranking is a root of a prefix tree. To traverse this trees in a depth-first manner, DeepTESMA (Algorithm \ref{algo:DeepTESMA}) is recursively called.

This algorithm accepts as input a ranking $\rho$ (with last element $o_l = last(\rho)$), its transaction list $\ct_{\rho}$, and a depth counter. Then, DeepTESMA is recursively called for every $\rho \concat o$, where  $o \in keys(OT_l)$. A more detailed description is given in Algorithm \ref{algo:TESMA}.


\begin{algorithm}[t]
\SetAlgoNoLine
		\KwIn{$t$, data}
			$OOT \leftarrow HHM(data)$\label{alg:ln:a1l1}\tcp*[r]{(Hash Map) Global variable}
			$L \leftarrow \emptyset$\tcp*[r]{List (Global variable)}
			\For{$i \rightarrow |O|$}{
				$l_i \leftarrow \emptyset$\tcp*[r]{List}
				$L(i) \leftarrow \{l_{i}\}$\;
			}
			\For{$(o_i,OT_{o_i}) \in OOT$}{
				\For{$(o,t_{o_i \concat o}) \in OT_{o_i}$}{
					$l_1 \leftarrow l_{1} \cup \{(o_i \concat o)\}$\;
					$DeepTESMA((o_i \concat o), 2,t_{o_i \concat o})$\;
				}
			}
	\caption{TESMA}\label{algo:TESMA}
\end{algorithm}

		%
				%

\begin{algorithm}[t]
\SetAlgoNoLine
		\KwIn{$\rho \concat o_l,t_{\rho \concat o_l},k$}
			\For{$(o,t_{o_l \concat o}) \in OT_{o_l}$\label{alg:ln:a3l1}}{
				$t_{\rho \concat o_l\concat o} \leftarrow t_{\rho,\concat o_l}\cap t_{o_l \concat o}$\;
				\If{$|t_{\rho\concat o_l\concat o}| \geq \delta$}{
					$l_k \leftarrow l_k \cup \{\rho\concat o_l \concat o\}$\;
					$DeepTESMA(\rho \concat o_l \concat o, t_{\rho\concat o_l\concat o},k+1)$\;
				}
			}
	\caption{DeepTESMA}\label{algo:DeepTESMA}
\end{algorithm}

\subsection{Analysis of the algorithm}

TESMA is correct in the following sense: all orders found by the algorithm are frequent, and all frequent orders are found by the algorithm. Since this is quite easy to see, we refrain from proving the result in a formal way. Essentially, it follows from the combination of three observations. First, the computations of frequencies (support) of orders based on bit-vectors is correct. Second, depth-first search delivers a complete enumeration of the entire search space $\overline{\mathbb{S}}_K$. Third, our pruning procedure is sound, i.e., when a subtree is pruned by TESMA, it will definitely not contain any frequent order; obviously, this follows from the monotonicity property (\ref{eq:mono}). 

In order to analyze the time complexity of TESMA, recall that we denote by $K = |\mathbb{O}|$ the number of objects and by $N = | \mathbb{D} |$ the number of transactions in the database; moreover, let $L$ be the number of frequent rankings contained in $\mathbb{D}$. The first step in TESMA is to calculate the HHM (line \ref{alg:ln:a1l1}), which obviously needs $K^2 N$ operations. The rest of the algorithm is a depth-first search in the prefix-trees. Every discovered path starting at the root corresponds to a frequent ranking, which can be extended by adding a new object at the end. 
To expand a path, $K-1$ objects (successor nodes) need to be tested in the worst case   
(Algorithm \ref{algo:DeepTESMA}, line \ref{alg:ln:a3l1}). Hence, the depth-first search requires $L(K-1)$ operations. Overall, we have a running time of $\mathcal{O}(L(K-1) + K^2N)$.

As for the space complexity, the big advantage of TESMA over previous algorithms is that it avoids scanning the entire database for every candidate ranking. This of course comes with a disadvantage, namely the need to store transaction bit-vectors, each of which has a size of $\mathcal{O}(N)$. In the worst case, the HHM stores every possible 2-ranking and its transaction bit-vector, which requires $\mathcal{O}(K^2 N )$ space. In addition, we have to keep the transaction bit-vector of suborders on the current search path in main memory, the maximal length of which is  $K$.
Therefore, the overall space complexity of TESMA is $\mathcal{O}(K^2 N + KN ) = \mathcal{O}(K^2 N)$, plus the space needed to store the frequent rankings.

\section{Mining closed rank patterns}
\label{sec:closed}

In pattern mining, one is typically more interested in large patterns than in small ones. However, as shown by monotonicity properties such as (\ref{eq:mono}), the criterion of size is normally in conflict with the criterion of frequency. Nevertheless, given two patterns with the same frequency, it is natural to prefer the larger to the smaller one. This leads to the notion of \emph{closedness} as defined for rankings in (\ref{eq:mono}), and the idea to focus on closed patterns, i.e., patterns that cannot be extended without lowering support.

\subsection{Closures and generators of rankings}

In itemset mining, the closure of itemsets is obtained by means of the Galois operator \cite{Pasquier1999}. 
Let $g: \, 2^{\mathbb{O}} \longrightarrow 2^{\mathbb{D}}$ be the function that maps an itemset $O \subset \mathbb{O}$ to the set of all transactions $D \subset \mathbb{D}$ containing $O$.
Furthermore, let $f: \, 2^{\mathbb{D}} \longrightarrow 2^{\mathbb{O}}$ be the function that maps a set of transactions $D$ to their largest common itemset $\cap_{O_i  \in D} O_i$.
Then, $h: \, 2^{\mathbb{O}} \longrightarrow 2^{\mathbb{O}}$ defined as $h = f \circ g$ is the Galios operator for itemsets,
$h(O)$ is a closed itemset for any $O \subset \mathbb{O}$, and $h(h(O)) = h(O)$.

For reasons of efficiency, algorithms such as Aclose \cite{Pasquier1999} do not mine all frequent patterns but only their smallest generators. An itemset $O$ is a generator of another itemset $O'$ if the latter is the closure of the former.
As will be seen, transferring this idea from the case of itemsets to the case of rankings is not straightforward, for example because the Galois operator cannot be applied as easily anymore. In any case, we first of all have to clarify what we mean by the closure of a ranking. 

\begin{definition}
Given two subrankings $\pi_A$ and $\pi_B$, we call
\[
\pi_A \cap \pi_B \, = \, \big\{  \pi \in \overline{\mathbb{S}}_K \,\big\vert\, O(\pi) \subset O(\pi_A) \cap O(\pi_B), \, (\pi_A|O(\pi)) = \pi, (\pi_B|O(\pi)) = \pi \big\}
\]
their intersection, and 
\[
\pi_A \cap^* \pi_B \, = \, \big\{ \pi \in \overline{\mathbb{S}}_K \,\big\vert\,  
 \pi \in \pi_A \cap \pi_B, \, \nexists \pi' \in \pi_A \cap \pi_B: \, \pi \subsetneq \pi' \big\}
\]
their maximal intersection. 
\end{definition}
Obviously, the (maximal) intersection of rankings is a commutative and associative operation that can be extended to more than two rankings in a canonical way. 

\begin{definition}
Let $g: \, \bbS \longrightarrow 2^{\mathbb{D}}$ be the function that maps rankings $\pi \in \mathbb{D}$ to all transactions $\pi_i$ with $\pi \subset \pi_i$. Furthermore, let $f: \, 2^{\mathbb{D}} \longrightarrow 2^{\bbS}$ be the function that maps a set of transactions to their maximal intersection.
We then call
$$
h: \, \bbS  \longrightarrow 2^{\bbS}, \, \pi \mapsto  (f \circ g)(\pi)
$$
the h-closure, and $h(\pi)$ the h-closure of the ranking $\pi$.
\label{def:HClosure}   
\end{definition}
Based on this definition, the following result is rather straightforward.

\begin{lemma}
A ranking $\pi$ is closed in the sense of (\ref{eq:closed}) iff $\pi \in h(\pi)$.
\end{lemma}

\begin{figure}
	\centering
		\includegraphics[width = 0.7\textwidth]{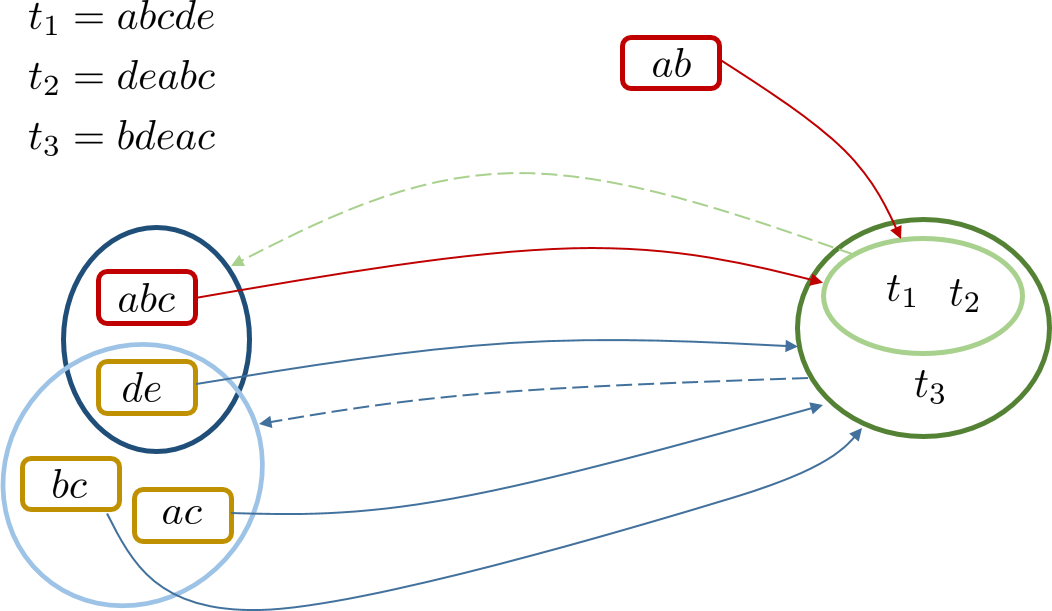}
		\caption{Here the g-closure and h-closure are applied several times starting with $(a,b)$. We have $h((a,b)) = \{(a,b,c),((d,e))\}$ with $h((a,b,c))\neq h((d,e))$ but $g((a,b,c)) \subset g((d,e))$.}
		\label{fig:Closure1}
\end{figure}
\begin{figure}
	\centering
		\includegraphics[width = 0.5\textwidth]{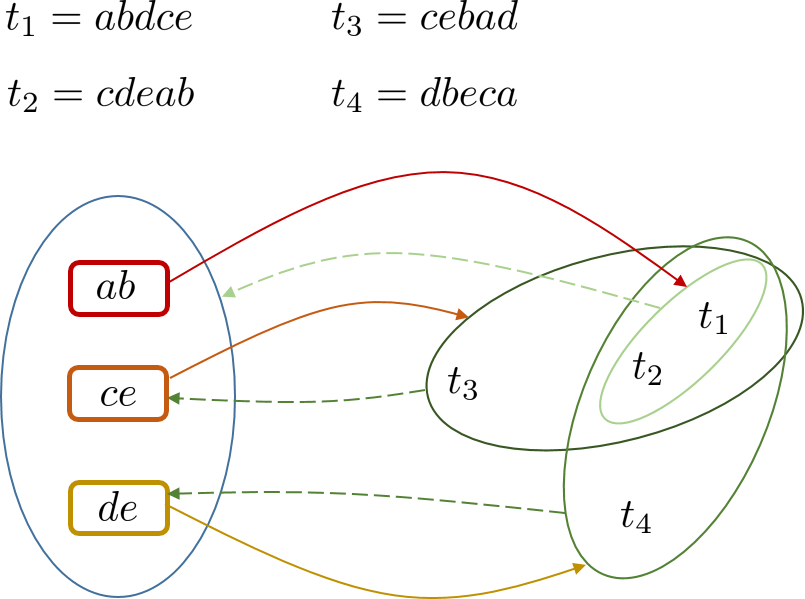}
		\caption{Here we have a situation where the g-closures are not nested: $h((a,b)) = \{(a,b),(c,e),(d,e)\}$ but $g((c,e)) \not\subset g((d,e))$ and $g((d,e)) \not\subset g((c,e))$.}
		\label{fig:ClosureH2}
\end{figure}

In contrast to itemset mining, where the closure of an itemset is again an itemset, the h-closure of a ranking $\pi$ is a set of rankings, and in general, not every ranking in this set contains $\pi$; on the contrary, $h(\pi)$ may even contain rankings $\pi'$ that do not share any item with $\pi$. Moreover, for $\pi' \in h(\pi)$, it is thoroughly possible that $h(\pi) \neq h(\pi')$; see Figures \ref{fig:Closure1} and \ref{fig:ClosureH2} for an illustration. Therefore, it is useful to introduce another notion of closure as follows. 
\begin{definition}
For a ranking $\pi \in \bbS$, we call 
\[
c(\pi) \, = \, \big\{  \pi'  \in h(\pi)  \,\vert\, \pi \subset \pi'   \big\}
\]
the c-closure of $\pi$. A ranking $\pi$ is called a generator of a closed ranking $\pi^*$, if $\pi^* \in c(\pi)$.
\end{definition}

\begin{proposition}
Consider any ranking $\pi \in \bbS$. All rankings $\pi' \in h(\pi)$ in the h-closure of $\pi$ are closed (including those with $h(\pi') \neq h(\pi)$).
\label{prop:2}
\end{proposition}
\emph{Proof.} To prove this proposition, we first show a property for those $\pi' \in h(\pi)$ with $h(\pi') \neq h(\pi)$. For such rankings, the following holds. First, $g(\pi) \subset g(\pi')$, because if $h(\pi') \neq h(\pi)$, then $\pi'$ needs to be contained in additional transactions. Furthermore, $g(\pi) \subset g(\pi'')$ for all $\pi'' \in h(\pi')$. Second, for all $\pi'' \in h(\pi')$, there exists some $\pi^* \in h(\pi)$ such that $\pi'' \subset \pi^*$. This holds because rankings contained in a maximal intersection of transactions can only grow if the set of transactions is reduced.  

Now, consider $\pi' \in h(\pi)$ and assume that $\pi'$ is not closed, which means that $\pi' \notin h(\pi')$. Thus, $\pi'$ is not maximal in $h(\pi')$, and hence not maximal in $f(D)$ for any $D \subset g(\pi')$; according to what we just showed, this especially applies to $D = g(\pi) \subset g(\pi')$, and  $\pi'$ is not maximal in $h(\pi)$. However, from the definition of $h$, it immediately follows that $\pi'$ is maximal in $h(\pi)$ for all $\pi' \in h(\pi)$, which is a contradiction. \hfill $\Box$

In general, a closed pattern may have more than one generator. This is why, for example, the AClose algorithm only stores generators that are smallest in the sense of a predefined lexicographic order on items. Since lexicographic orders cannot be used in the case of rankings, our approach is based on \emph{prefix generators}, where a prefix of a ranking $\rho = (o_{i_1}, o_{i_2}, \ldots , o_{i_k})$ is a subranking of the form  $(o_{i_1}, o_{i_2}, \ldots , o_{i_p})$ for some $p \leq k$ (likewise, a suffix of $\rho$ is of the form $(o_{i_j}, \ldots , o_{i_k})$ for some $j \geq 1$).  Prefix generators are prefixes of c-closures that are also generators of that c-closure. In this regard, the following observation is important.

\begin{lemma}
Suppose that $\rho' \in c(\rho_p)$, and that $\rho_p$ is a prefix of $\rho'$. Then $\rho_p$ is a prefix for all $\rho \in c(\rho_p)$.
\label{lem:algo2}
\end{lemma}
\emph{Proof.} 
Suppose the claim is not correct, i.e., suppose there is a closed ranking $\rho \in c(\rho_p)$ such that $\rho_p$ is not a prefix of $\rho$. Then, it follows that all transactions $g(\rho_p)$ contain an object $o' \in O(\rho) \setminus O(\rho_p)$ with $o' \succ o_{i_p}$. This, however, means that $\rho'$ can not be in $c(\rho_p)$, which is a contradiction. \hfill $\Box$

Finally, we define the postfix of a c-closure, which is used to speed up prefix-tests in the algorithm.
\begin{definition}
Given a generator and prefix $\rho \concat o$ of $c(\rho \concat o)$, the subtree with $o$ as root is called postfix of $c(\rho \concat o)$ under $\rho \concat o$.
\end{definition}
An example of h-closure, c-closure, prefix and postfix is given in Figure \ref{fig:closureTypes}.

\begin{figure}
	\centering
		\includegraphics[width = 0.9\textwidth]{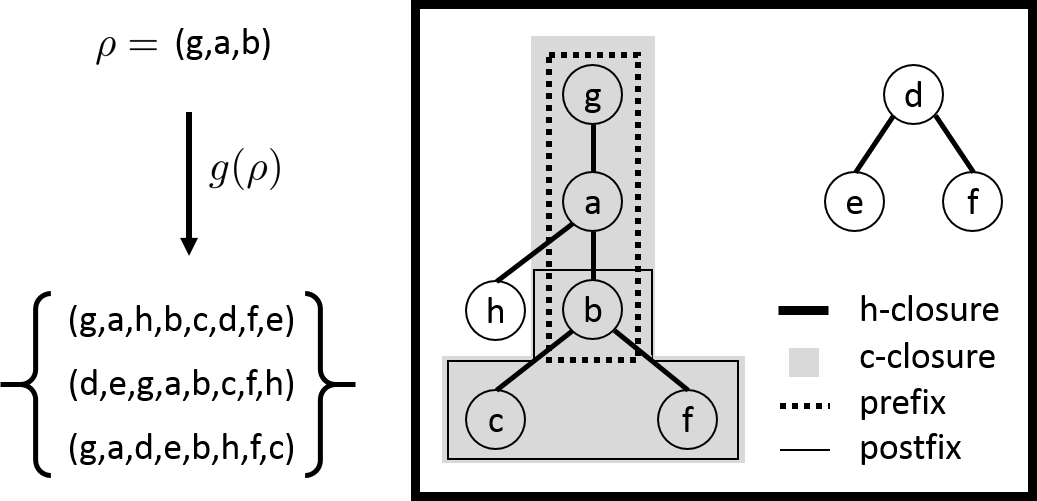}
		\caption{Illustration of the h-closure, c-closure, prefix and postfix of order $(g,a,b)$.}
		\label{fig:closureTypes}
\end{figure}

\subsection{GPMiner}
\label{sec:closedAlgo}

Our algorithm \emph{Generating Prefix Miner} (GPMiner) for mining closed rankings is based on TESMA and exploits Lemma  \ref{lem:algo2} to prune subrankings that do not constitute a prefix of their c-closure. Thus, for a new subranking built by tail extension, GPMiner not only checks the frequency but also whether this ranking is a prefix of its c-closure---compared to TESMA, this is the main difference of the algorithm.
GPMiner traverses the prefix search space in a depth first manner.
Given a candidate subranking $\rho|o$ (generated by tail extension) the algorithm follows three steps:

\begin{enumerate}
\item \textbf{Frequency check:} Like in TESMA, $\rho|o$ is frequent if $g(\rho) \cap g(last(\rho)|o)$ is frequent.
\item \textbf{Computing h-closure:} If $\rho|o$ is frequent, the h-closure $h(\rho|o)$ is computed.
\item \textbf{Prefix test:} At least $\rho|o$ has to be a prefix of its c-closure.
Due to implementation issues, the test is conducted on the h-closure (see Section \ref{sec:prefixtest} for details).
If $\rho|o$ is not a prefix, the subtree with $\rho|o$ as a root is pruned.
\end{enumerate}

To understand the main procedure (Algorithm \ref{algo:deepGPMiner}), we present an example based on the database in Table \ref{tab:exDataSet}; the example is also illustrated in Figure \ref{fig:GPMinerBeispiel}. We assume an absolute frequency threshold of $2$. Suppose we already know that $\rho|o_l = (a,b)$ is frequent. We select the first possible extension and generate the candidate $\rho|o_l|o = (a,b,c)$ (line \ref{alg:ln:a5l1}).
We find that $(a,b,c)$ is frequent (line \ref{alg:ln:a5l3}), since $g(\rho | o_l) \cap g(o_l|o) = g((a,b)) \cap g((b,c)) = \{t_1,t_2\} \cap \{t_1,t_2,t_4\} = \{t_1,t_2\}$ and $\#\{t_1,t_2\} = 2$.
Now we have to determine if $g((a,b)) = g((a,b,c))$ (line \ref{alg:ln:a5l4}).
If so, then $(a,b)$ is not closed and has to be deleted from the result list.
Here, however, $g((a,b)) = \{t_1,t_2,t_3\} \neq \{t_1,t_2\} = g((a,b,c))$, so we continue with computing the h-closure $h((a,b,c))$ (lines \ref{alg:ln:a5l12} and \ref{alg:ln:a5l13}), which is $\{(a,b,c),(a,b,e),(a,d)\}$.
We perform the prefix test on $(a,b,c)$ (line \ref{alg:ln:a5l14}).
Thanks to Lemma \ref{lem:algo2}, we know that if $(a,b,c)$ is a prefix of one element in $c((a,b,c))$, then $(a,b,c)$ is a prefix of all elements in $c((a,b,c))$.
Since in general $c(\rho) \subseteq h(\rho)$, we have to find one ranking in $h((a,b,c)) = \{(a,b,c),(a,b,e),(a,d)\}$ for which $(a,b,c)$ is a prefix. In this case, we find $(a,b,c)$, thus $(a,b,c)$ is a generating prefix and the algorithm continues with the new candidate $(a,b,c,d)$.
However, $(a,b,c,d)$ does not pass the frequency test and its subtree is pruned.
Likewise, $(a,b,c,e)$ is not frequent and therefore pruned.
The next candidate $(a,b,d)$ is frequent but does not pass the prefix test.
Hence we continue with $(a,b,e)$.
It passes the frequency test and also the prefix test.
In addition, $\#g((a,b,e)) = \#g((a,b))$.
Hence $(a,b)$ can not be closed and is deleted from the result list after all possible extensions (i.e., paths to child nodes) have been explored (line 24).
The candidate $(a,b,e,c)$ does not pass the frequency test either, whereas $(a,b,e,d)$ is frequent and a prefix of its c-clousure. The last candidate $(a,b,e,d,c)$ for the $(a,b)$-tree is again not frequent.
Thus, the mining process results in the closed rankings $\{(a,b,c),(a,b,e),(a,b,e,d)\}$ for the $(a,b)$-tree.
To complete the mining procedure, the trees for all other frequent pairs have to be traversed, too.

\begin{figure}
	\centering
		\includegraphics[width = 0.9\textwidth]{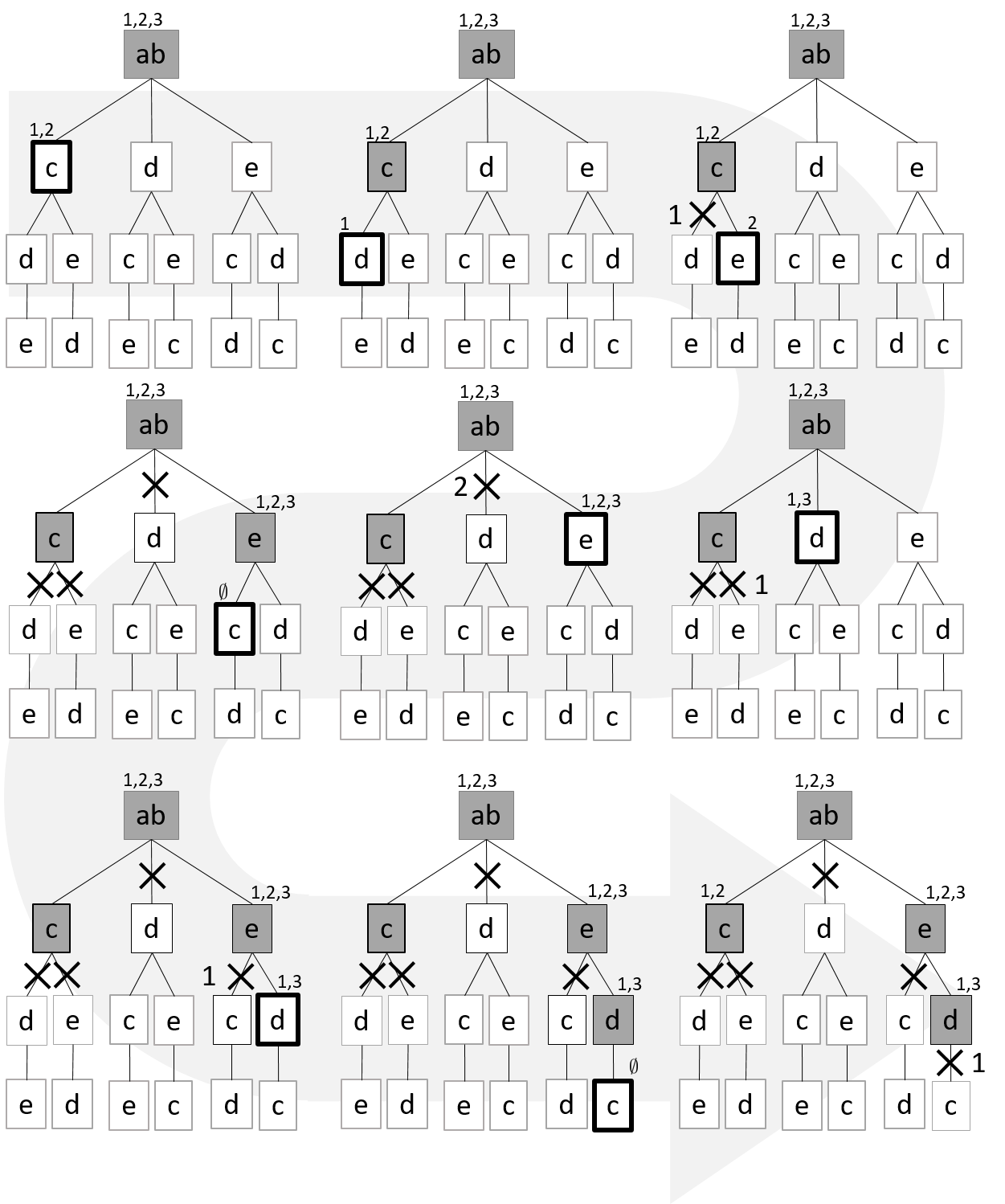}
		\caption{Mining for closed subrankings in the $(a,b)$-tree based on the data in Table \ref{tab:exDataSet}. The small numbers above the nodes indicate the transactions the subranking is contained in. The number 1 next to a cross indicates a pruning step because of infrequency, whereas 2 indicates pruning because of not being a generating prefix. The mining process of the $(a,b)$-tree results in the closed rankings $\{(a,b,c),(a,b,e),(a,b,e,d)\}$.}
		\label{fig:GPMinerBeispiel}
\end{figure}

The Algorithm \ref{algo:GPMiner} is responsible for initializing the required data structures and variables and performs the first mining step. The following list gives an overview of the data structures and variables:
\begin{itemize}
\item[$OOT$] Data structure to hold the frequent pairs together with their g-closures (Algorithm \ref{algo:HHM}).
\item[$s1p$] support-1-preference matrix required to compute the h-closure of a ranking
(Algorithm \ref{algo:s1p}).
\item[$lp_k$] A set of 2-rankings used to update the $s1p$ matrix during depth search (Algorithm \ref{algo:s1p}).
\item[$\ct$] A bit vector representing a set of transactions (rankings).
\item[$\ct_{\rho}$] A bit vector representing the g-closure $g(\rho)$.
\item[$l_k$] Set to store all closed subrankings of length $k$.
\item[$L$] Global set that stores the sets $l_k$.
\end{itemize}


\begin{algorithm}
\SetAlgoNoLine
		\KwIn{data}
		\KwOut{A List $L$ of lists, whereas each list stores rankings of the same size.}
			$OOT \leftarrow HHM(data)$\tcp*[r]{Global variable}
			$L \leftarrow \emptyset$\tcp*[r]{List (Global variable)}
			\For{$i \rightarrow |O|-1$}{
				$l_i \leftarrow \emptyset$\;
				$L(i) \leftarrow \{l_{i}\}$\;
			}
			\For{$(o_i,OT_{o_i}) \in OOT$}{
				\For{$(o,t_{o_i\concat o}) \in OT_{o_i}$}{
					$lp_1 \leftarrow \emptyset$\tcp*[r]{List of pairs in k-th depth level}
					$lp_1 \leftarrow lp_1 \cup \{o_i \concat o\}$\;
				}
			}
			\For{$(o_i,OT_{o_i}) \in OOT$}{
				\For{$(o,t_{o_i\concat o}) \in OT_{o_i}$}{
					$[s1p, lp_2]\leftarrow  S1P\_MATRIX(t_{o_i \concat o},lp_1)$\tcp*[r]{$s1p$: Matrix representing the h-closure}
					$hClosure \leftarrow HCLOSURE(t_{o_i\concat o},s1p)$\;
					$postfix \leftarrow R\_PREFIX\_TEST(hClosure,o_i \concat o)$\;
					\If{$postfix \neq \emptyset$}{
						$l_1 \leftarrow l_{1} \cup \{(o_i \concat o)\}$\;
						$DeepGPMiner(o_i \concat o,t_{o_i \concat o},postfix,s1p,lp_2,2)$\;
					}
				}
			}
	\caption{GPMiner}\label{algo:GPMiner}
\end{algorithm}

						%
						%

\begin{algorithm}
\SetAlgoNoLine
		\KwIn{$\rho \concat o_l,\ct_{\rho\concat o_l},postfix, s1p, lp_k, k$}						
			\For{$(o,\ct_{o_l \concat o}) \in OT_{o_l}$\label{alg:ln:a5l1}}{
				$\ct_{\rho \concat o_l \concat o} \leftarrow \ct_{\rho\concat o_l}\wedge \ct_{o_l \concat o}$\;
				\If{$|\ct_{\rho \concat o_l \concat o}| \geq \delta$\label{alg:ln:a5l3}}{
				\tcp{$\delta$ is the absolute frequency threshold}
					\eIf{$\ct_{\rho \concat o_l \concat o} = \ct_{\rho \concat o_l}$\label{alg:ln:a5l4}}{
						$currPostfix \leftarrow R\_PREFIX\_TEST(\rho\concat o_l\concat o, postfix, |O(\rho\concat o_l)|)$\;
						
						\If{$currPostfix \neq \emptyset$}{
							$l_k \leftarrow l_k \cup \{\rho\concat o_l \concat o\}$\label{alg:ln:a5l7}\;
							$longerFound \leftarrow true$\;
							$DeepGPMiner(\rho\concat o_l\concat o,\ct_{\rho\concat o_l\concat o} ,postfix, s1p, lp_{k+1}, k+1)$\;
						}}{
						$[s1pSucc, lp_{k+1}] \leftarrow S1P\_MATRIX(\ct_{\rho\concat o_l \concat o},lp_k,s1p)$\label{alg:ln:a5l12}\;
						$hClosure \leftarrow HCLOSURE(t_{\rho\concat o_l \concat o},s1p)$\label{alg:ln:a5l13}\;
						$currPostfix \leftarrow PREFIX\_TEST(hClosure,\rho\concat o_l \concat o)$\label{alg:ln:a5l14}\;
						\If{$currPostfix \neq \emptyset$}{
							$l_k \leftarrow l_k \cup \{\rho\concat o_l \concat o\}$\;
							$DeepGPMiner(\rho\concat o_l \concat o, \ct_{\rho \concat o_l \concat o},postfix, s1pSucc, lp_{k+1}, k+1)$\;
						}
					}
				}
			}
			\If{$longerFound$}{
				$l_{(k-1)} \leftarrow l_{(k-1)} \cap \rho\concat o_l$\label{alg:ln:a5l23}\;
			}
	\caption{DeepGPMiner}\label{algo:deepGPMiner}
\end{algorithm}

\subsubsection{Computing the h-closure}

As already mentioned, the main difference between TESMA and GPMiner is a pruning step, in which all subrankings that are not prefix of their c-closure are pruned. In practice, the prefix test is not done on the c-closure but on the h-closure. This is possible, since if there is a $\rho' \in h(\rho)$ such that $\rho$ is prefix of $\rho'$, then $\rho$ is prefix of $c(\rho)$ according to Lemma \ref{lem:algo2}. The h-closure $h(\rho)$ is in principle nothing else than the maximal intersection $\bigcap_{\rho' \in g(\rho)}^* \rho'$. But computing the intersection of two rankings is not as trivial as it is for sets. One way is to divide the operation into two steps. The first is to apply an ordinary set intersection between the sets of 2-rankings representing the transactions (Algorithm \ref{algo:s1p}). For instance, $\{(a,b),(a,c)\}$ is the set intersection between $a \succ b \succ c$ represented as $\{(a,b),(a,c),(b,c)\}$ and $a \succ c \succ b$ represented as $\{(a,b),(a,c),(c,b)\}$. As a result, we obtain the h-closure represented by a set of 2-rankings. This set representation is not suitable for performing the prefix test, however. Therefore, in a second step, the set representation is transformed into a tree representation (Algorithm \ref{algo:hclosure}).

The first step to compute the h-closure $h(\rho)$ for some candidate $\rho$ is to compute the intersection of all transactions (rankings $\rho_i$) contained in the g-closure $g(\rho)$. This is done by the algorithm S1P\_MATRIX (Algorithm \ref{algo:s1p}). To do so efficiently, it takes advantage of the information stored in $OOT$. Given a 2-ranking $(i,j)$ and its transaction vector $\ct_{i,j}$, $(i,j)$ is shared by all transactions in $\ct_{\rho}$ if $\ct_{i,j}$ AND $\ct_{\rho} = \ct_{\rho}$. To compute the intersection, this test is applied to each of the $|O|^2$ 2-rankings. One advantage is that $|O|$, the number of objects, is normally quite small compared to the size of the g-closure. 

Another advantage becomes obvious from the following considerations: If the ranking $\rho$ is frequent and a prefix of its c-closure, $\rho$ will be extended in the next level of recursion (Algorithm \ref{algo:deepGPMiner}). Let $\rho \concat o$ be one of the possible extensions. We know that $g(\rho \concat o) \subset g(\rho)$ and hence $h(\rho) \subset h(\rho \concat o)$. This means that all 2-rankings in $h(\rho)$ are also in $h(\rho \concat o)$. Consequently, in a deeper recursion level, only the 2-rankings $(i,j) \notin h(\rho)$ have to be tested. These 2-rankings are stored in the set $lp_k$, where $k$ indicates the current recursion depth. The result of an intersection is stored in an $|O| \times |O|$ matrix called \textit{support-1-preference matrix} ($s1p$-matrix). An entry $e_{j,i}$ in this matrix is $1$ if $(i,j) \subset h(\rho)$ and $0$ otherwise. 


\begin{algorithm}
\SetAlgoNoLine
		\KwIn{$\ct$,$lp$,$s1p$}
		\KwOut{$s1p$, $lp_r$}
			$lp_r \leftarrow lp$\;
			\For{$(i,j) \in lp$}{
				\If{$\ct_{i|j} \supset \ct$}{
					$s1p[j][i] \leftarrow 1$\tcp*[r]{$s1p[rows][columns]$}
					$lp_r \cap \{(i,j)\}$\;
				}
			}
	\caption{S1P\_MATRIX}\label{algo:s1p}
\end{algorithm}

\subsubsection{Construction of trees}

After the $s1p$-matrix is computed, it needs to be converted into a set of trees, for which the HCLOSURE procedure is responsible (Algorithm \ref{algo:hclosure}). In principal, the $s1p$-matrix can be seen as a representation of a directed graph $G_{s1p} = (V,E)$. Since a path in this graph defines a ranking, we will subsequently denote them by $\rho$. The task is to find all longest paths $\rho$, i.e., those paths for which no other path $\rho'$ exists such that  $\rho \subset \rho'$. The basic procedure to find longest paths is as follows (see Algorithm \ref{algo:hclosure} for details):
\begin{enumerate}  
	\item Identify all sources (i.e., nodes with only outgoing edges).
	\item For each source, span a tree and traverse this tree in a depth-first manner.
	\begin{enumerate}
\item Given a node $v \in V$, let $\mathcal{A}(v)$ be the set of nodes on the path from the root (which is included) to the parent node of $v$ (which is excluded). If $v$ is the current node and $y(v)$ its value, delete all nodes $w$ with $w = child(u)$, $u \in \mathcal{A}(v)$ and $y(v) = y(w)$.
	\end{enumerate}
\end{enumerate}
An example is given in Figure \ref{fig:Intersec}. If $\rho$ is the current candidate, a ranking $\rho_1$ is first selected from the set of transactions $ \ct_{\rho} = g(\rho)$ (Algorithm \ref{algo:hclosure}, line \ref{alg:ln:a7l1}).
Ideally, $\rho$ is the smallest ranking in $\ct_{\rho}$, but since assurance of this property requires sorting of all rankings, we simply pick the first ranking in the original order. Now, $\rho$ is used to permute the rows and columns of the $s1p$-matrix such that all 1-entries are in the lower triangular matrix.
Keep in mind that the diagonal of $s1p$ is filled with zeros at the beginning of the procedure.
If a node is visited during the depth first search, the entry for this node on the diagonal will be changed to $1$ (Algorithm \ref{algo:ncr}, line \ref{alg:ln:a8l13}). This excludes the node from the set of possible sources (the roots). 
In lines \ref{alg:ln:a7l4}--\ref{alg:ln:a7l6} (alg. \ref{algo:hclosure}), every object $o \in O(\rho_1)$ is wrapped into a node, which are collected in a list.
All nodes are possible roots for a tree and so the algorithm iterates through all nodes performing the tree construction process.
The first node associated with the first column and therefore with the object ranked first in $\rho$ must be a root of one of the final trees, since a longest path always starts with a source node, i.e., a node with indegree 0.
If the root candidate has not already been visited (algo. \ref{algo:hclosure} line \ref{alg:ln:a7l8}), it is added as a root to the set $hClosure$, which at the end contains the whole forest.
In line \ref{alg:ln:a7l10}, $\mathcal{A}[1]$ is assigned with the index of the corresponding object.
This is needed to determine the rows the entries of which should later be changed from 1 to 0.
Thus, changing an entry from 1 to 0 in the $s1p$-matrix is equal to deleting an edge in the $G_{s1p}$ graph (see Figure \ref{fig:Intersec}).
Finally, the recursive tree construction process is started (line \ref{alg:ln:a7l11}).


\begin{algorithm}
\SetAlgoNoLine
		\KwIn{$\ct$,$s1p$}
		\KwOut{$hClosure$}
			select $\rho \in \ct$\label{alg:ln:a7l1}\;
			$hClosure \leftarrow \emptyset$\tcp*[r]{list of treeNodes}
			$N_{List} \leftarrow empty$\;
			\For{$i \in |O(\rho)|$\label{alg:ln:a7l4}}{
				$N_{List} \leftarrow N_{List} \cup Node(\rho(i))$\;
			\label{alg:ln:a7l6}}
			\For{$i \in |O(\rho)|$}{
				\If{$s1p[\rho(i)][\rho(i)] = 0$\label{alg:ln:a7l8}}{
					$hClosure.add(N_{List}(\rho(i)))$\;
					$\mathcal{A}[1] \leftarrow \rho(i)$\label{alg:ln:a7l10}\;
					RECURSIVE\_CONNECT($i,A,s1p,N_{List},\rho$)\label{alg:ln:a7l11}\;
				}
			}
	\caption{HCLOSURE}\label{algo:hclosure}
\end{algorithm}

The procedure \textit{RECURSIVE\_CONNECT} iterates through the rows (one nodes children), starting with row $\rho(p+1)$ in column $\rho(p)$ (line \ref{alg:ln:a8l1}).
If a cell with entry $1$ is hit (line \ref{alg:ln:a8l2}), it is checked whether the corresponding node has already been visited (line \ref{alg:ln:a8l3}).
If not, it is added as a child to node $\rho(p)$ (line \ref{alg:ln:a8l4}).
Now, all entries in the current row of the ancestors are set to $0$ (lines \ref{alg:ln:a8l6}--\ref{alg:ln:a8l8}).
Let $\rho(j)$ be the current node (row).
Then, changing the entries for the ancestors from 1 to 0 is equivalent to deleting all of the ancestor's child nodes with the same value. 
Afterward, the index $i$ is added to the ancestors (line \ref{alg:ln:a8l9}).
In line \ref{alg:ln:a8l10}, the procedure calls itsself with updated parameters.
Finally, the considered object is marked as visited.


\begin{algorithm}
\SetAlgoNoLine
		\KwIn{$p,\mathcal{A},s1p,N_{List},\rho$}
			\For{$i = p+1 \to |O(\rho)|$\label{alg:ln:a8l1}}{
				\If{$s1p[\rho(i)][\rho(p)] = 1$\label{alg:ln:a8l2}}{
					\If{$s1p[\rho(i)][\rho(i)] = 0$\label{alg:ln:a8l3}}{
						$N_{List}(\rho(p)).addAsChild(N_{List}(\rho(i)))$\label{alg:ln:a8l4}\;
					}
					\For{$j=1 \to |\mathcal{A}|$\label{alg:ln:a8l6}}{
						$s1p[\rho(i)][\rho(\mathcal{A}[j])] \leftarrow 0$\;
					}\label{alg:ln:a8l8}
					$\mathcal{A}[|\mathcal{A}|] = \rho(i)$\label{alg:ln:a8l9}\;
					RECURSIVE\_CONNECT($i,\mathcal{A},s1p,N_{List},\rho$)\label{alg:ln:a8l10}\;
				}
			}
			$s1p[\rho(p)][\rho(p)] \leftarrow 1$\label{alg:ln:a8l13}\;
	\caption{RECURSIVE\_CONNECT}\label{algo:ncr}
\end{algorithm}

\begin{figure}
	\centering
		\includegraphics[width = 1.0\textwidth]{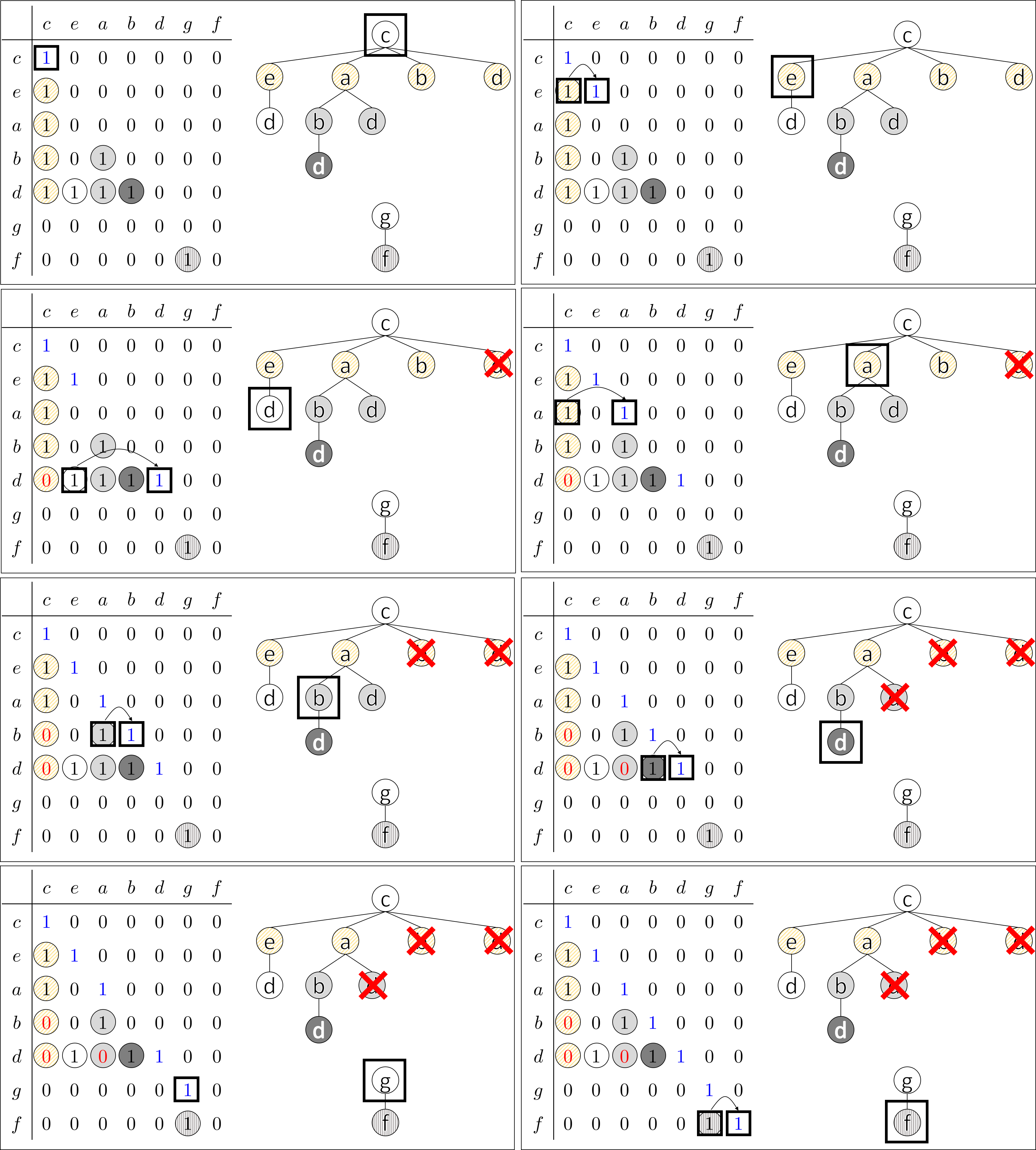}
		\caption{Illustration of the HCLOSURE algorithm. On the left side of every subfigure the s1p matrix is illustrated, on the right side the corresponding tree. The node framed in black is the node currently considered during the depth-first search.}
		\label{fig:Intersec}
\end{figure}

\subsubsection{Prefix test}
\label{sec:prefixtest}

Testing if a subranking $\rho \concat o$ is a prefix of its c-closure can be done in a depth-first manner. Two cases need to be considered. First, in the case where $g(\rho\concat o) = g(\rho)$, only algorithm R\_PREFIX\_TEST($\rho \concat o$, $last(\rho)$, $|O(\rho)|$) has to be invoked (see Algorithm \ref{algo:recPostfixTest}). This is a recursive algorithm with termination condition $index > |O(\rho \concat o)|$ (line \ref{alg:ln:a9l1}). Unless the termination condition is met, every child of $root$ is tested for equality with $o$ (lines \ref{alg:ln:a9l4}--\ref{alg:ln:a9l5}), and the return value by success is the result of R\_PREFIX\_TEST($\rho \concat o$, $child$, $|O(\rho)| + 1$).  For the case $g(\rho \concat o) = g(\rho)$, the termination condition will already be fulfilled in the second call. The overall result is $o$ as new postfix root, if $\rho \concat o$ is prefix and $\rho \concat o$ will be stored together with its c-closure postfix, else the result is null and $\rho \concat o$ will be pruned.


\begin{algorithm}
\SetAlgoNoLine
		\KwIn{$\rho, root, index$}
			\If{$index$ $> |O(\rho)|$\label{alg:ln:a9l1}}{
				return $root$\;
			}
			\For{$child$ $\in$ $root$\label{alg:ln:a9l4}}{
				\If{$child$ $= \rho(index)$\label{alg:ln:a9l5}}{
					return R\_PREFIX\_TEST($\pi$,$child$,$index+1$)\;
				}
			}
			return null\;
	\caption{R\_PREFIX\_TEST}\label{algo:recPostfixTest}
\end{algorithm}

In the second case, where $g(\rho|o) \neq g(\rho)$, the c-closure or rather the h-closure has to be computed. The prefix test on the whole h-closure is in principle also done with R\_PREFIX\_TEST. But prior to this, PREFIX\_TEST (Algorithm \ref{algo:postfixTest}) joins the individual trees in hClosure into one single tree (lines \ref{alg:ln:a10l2}--\ref{alg:ln:a10l3}) and calls R\_PREFIX\_TEST with $index = 0$ (line \ref{alg:ln:a10l5}).


\begin{algorithm}
\SetAlgoNoLine
		\KwIn{$hClosure, \rho$}
			$root \leftarrow \emptyset$\;
			\For{$tree$ $\in hClosure$\label{alg:ln:a10l2}}{
				$root.addChild(tree)$\label{alg:ln:a10l3}\;
			}
			return R\_PREFIX\_TEST($\rho$,$root$,1)\label{alg:ln:a10l5}\;
	\caption{PREFIX\_TEST}\label{algo:postfixTest}
\end{algorithm}

\subsection{Analysis of the algorithm}

To show the correctness of GPMiner, we need to show that it finds all closed rankings and that all rankings found by the algorithm are closed. 

We prove the first property by contradiction. Thus, suppose there is a closed ranking $\rho$ that the algorithm was not able to find. This can only happen if one of the prefixes, say, $\rho_p$ was either pruned or never found at all. A prefix is pruned if it is not frequent or not a prefix of its c-closure. If $\rho$ is frequent, then $\rho_p$ is obviously frequent, too, since $\rho_p \subset \rho$. Hence, $\rho_p$ can not be pruned because of frequency. In the case of $\rho \in c(\rho_p)$, $\rho_p$ is a prefix of its c-closure. But it can also happen that $\rho \notin c(\rho_p)$. If in this case $\rho_p$ is not a prefix of $c(\rho_p)$, $\rho_p$ can not be a prefix of $\rho$ either, since $\rho' \subset \rho$ for at least one $\rho'\in c(\rho_p)$. Thus, pruning $\rho_p$ because not being a prefix of its c-closure is again not possible. 

Now, consider the second reason for why $\rho$ might not be found, namely that $\rho_p$ has already not been found.
Obviously, the 2-prefix of $\rho$ is found and stored by GPMiner, even if it is not a generator of $\rho$ (since it is a generator and prefix of some other closed ranking). Consequently, there is a smallest $k>2$ for which the $k$-prefix of $\rho$ was not found. In other words, while the $(k-1)$-prefix has still been found by the algorithm, the $k$-prefix was not. However, this is not possible due to the correct extension of the algorithm from $k-1$ to $k$: GPMiner traverses the prefix-tree in a depth-first manner. Thus, if the $(k-1)$-prefix of $\rho$ is discovered and not pruned, the algorithm will extend it to the $k$-prefix of $\rho$ (Algorithm \ref{algo:deepGPMiner}, line \ref{alg:ln:a5l7}). Since $\rho$ is a frequent pattern, its $k$-prefix is frequent, too, and a prefix by definition. Hence, the $k$-prefix is definitely not pruned, which contradicts our assumption.

As for the second property, stating that only closed rankings are found, note that if $\rho$ is a non-closed ranking that is still considered by the algorithm, then it must be frequent and a prefix of a closed ranking---otherwise it would have been pruned. Now, since $\rho$ is a prefix, it will be extended. Moreover, since $\rho$ is not closed, one if its extensions will have the same frequency as $\rho$, hence it will be deleted after all possible extensions have been explored (Algorithm \ref{algo:deepGPMiner}, line \ref{alg:ln:a5l23}).

In order to analyze the time complexity of GPMiner, recall that we denote by $K$ be the number of objects and by $N$ the number of transactions. Moreover, let $L$ be the number of frequent rankings contained in the database. In principal, the only difference between TESMA and GPMiner is the pruning of all frequent subrankings that are not a prefix and generator at the same time. The effort to decide if a subranking should be pruned or not is the sum of the costs of the three subroutines S1P\_MATRIX, HCLOUSURE, and PREFIX\_TEST. Obviously, the cost of S1P\_MATRIX is $K^2$ in the worst case.
The cost for HCLOSURE is less than $K^2$,  in the $s1p$ matrix.
For estimating the complexity of HClOSURE we estimate how often an entry in the $S1P$-MATRIX is visited during the algorithm.
We distinguish between the entries on the diagonal and the entries below the diagonal.
An entry on the diagonal is visited at most $K-1$ times, which is the maximal number of preceding objects (nodes).
Since we have $K$ nodes, there are less than $K^2$ accesses to the diagonal entries.
Furthermore, the algorithm processes every entry below the diagonal only once  (through the children of the corresponding node) to determine if the entry is $0$ (no child) or $1$ (child).
In total, these are at most $K-1$ accesses (a node can at most have $K-1$ children). For $K$ nodes, we again obtain less than $K^2$ accesses altogether.
In addition, it is possible that an entry of the $S1P$-matrix is changed from $1$ to $0$.
This happens only once, however, and since there are at most $K(K-1)/2$ entries 1, we have again less than $K^2$ accesses.
In total, HCLOSURE requires less than $3K^2$ steps in the worst case.
Likewise, the cost of PREFIX\_TEST is upper-bounded by $K^2$: while traversing the ``hClosure-tree'', the PREFIX\_TEST procedure performs at most $K$ steps on each level and passes through at most $K$ levels. Summarizing, the pruning costs are of the order $\mathcal{K}^2$. In the worst case, every frequent ranking is also a closed ranking. Then, the running time for GPMiner is $\mathcal{O}(LK^3 + K^2 N)$.

\section{Experiments}
\label{sec:results}

In this section, we present a number of experiments with both synthetic and real data. These experiments are meant to analyze the behavior of TESMA and GPMiner with regard to properties such as dataset size ($N$), ranking size ($K$), and number of frequent patterns. Moreover, we like to demonstrate the advantages of our methods for mining rank data in comparison to methods for sequence mining. As already explained, such methods can in principle also be used, although they are not specialized to rank data.

To substantiate the usefulness of GPMiner for mining closed rankings, we also implemented another baseline called Post\_TESMA. This algorithm first calls TESMA to generate all frequent rankings, and then extracts all closed ones in a (fast) postprocessing step. 
There are two basic strategies to extract all closed patterns from the set of all frequent patterns: bottom-up and top-down.
The frequent patterns mined by TESMA are stored in a list of lists of orderings. The $i^{th}$ list of orderings stores all rankings of length $i+1$ (remember the smallest subranking is of length $2$). Due to the operating mode of TESMA, the lists of orderings are sorted lexicographically.
In the bottom-up strategy, an $(i+1)$-ranking $\pi$ of list $i$ is taken and compared with the $(i+2)$-rankings stored in list $i+1$ until an $(i+2)$-ranking $\pi' \supset \pi$ with the same frequency is found or the end of list $i+1$ is reached.
In the top-down strategy, an $(i+1)$-ranking $\pi'$ is compared with all $(i-1)$-rankings.
If an $(i-1)$-ranking $\pi \subset \pi'$ has the same frequency, it is deleted.
Practically, the bottom-up strategy was found to be significantly faster. This is mainly due to the fact that this strategy effectively exploits the data structure for storing the frequent patterns, especially the lexicographic order of the rankings.

All experiments in this section were done with $11$ repetitions. The first three runs where taken as warm up runs to acclimatize the Java JIT. The results were then obtained by averaging over the rest.

\subsection{Synthetic data}


An important advantage of synthetic data generation is that it allows for controlling certain properties and characteristics of the data.  
In a first experiment, we iteratively produced ranking datasets of constant size but with an increasing number of frequent (and closed) rankings. The generation process is as follows:
\begin{enumerate}
\item Generate a basic dataset of size $N$.
	\begin{enumerate}
		\item[a.] Generate $n$ rankings of size $K$ at random. These rankings constitute the core of the basic dataset.
		\item[b.] Replicate every core ranking $N/n$ times.
		\item[c.] Randomize the dataset by swapping every neighbored pair with a probability of $p$.
	\end{enumerate}
\item Mine all frequent rankings of the basic dataset with a frequency threshold of $\delta$, and store them in a list of lists $Fr\_LL$. The list $Fr\_LL(i)$ stores all rankings of size $i-1$.
\item Extract iteratively datasets from the mining result.
	\begin{enumerate}
		\item[a.] Set $i=1$, $j=1$, and $D_0 = \emptyset$.
		\item[b.] Chose a ranking $\pi$ of list $Fr\_LL(i)$ at random, remove it, and add $\pi$ to the dataset $D_{j-1}$ to produce the dataset $D_j$.
		\item[c.] Delete all subrankings of $\pi$ in $D_j$.
		\item[d.] Replicate $D_j$ until it has a size of $N$.
		\item[e.] If $Fr\_LL(i)$ is empty, set $i = i+1$.
		\item[f.] Set $j= j+1$ and continue with step a.
	\end{enumerate}
\end{enumerate}

The output produced by the algorithm is a set of datasets $\{D_1,\ldots,D_F\}$, where $F$ is the amount of frequent rankings mined from the basic dataset.
The dataset $D_i$ contains $i$ frequent rankings when mined with a frequency threshold of $1/(N+1)$.
The following settings were used for the experiment: $N = 100000$, $K=14$, $n = 4$, $p=0.1$, $\delta=0.01$.

\begin{figure}
\captionsetup[subfigure]{labelformat=empty}
	\centering
	\subfloat{
		\includegraphics[width=0.44\textwidth]{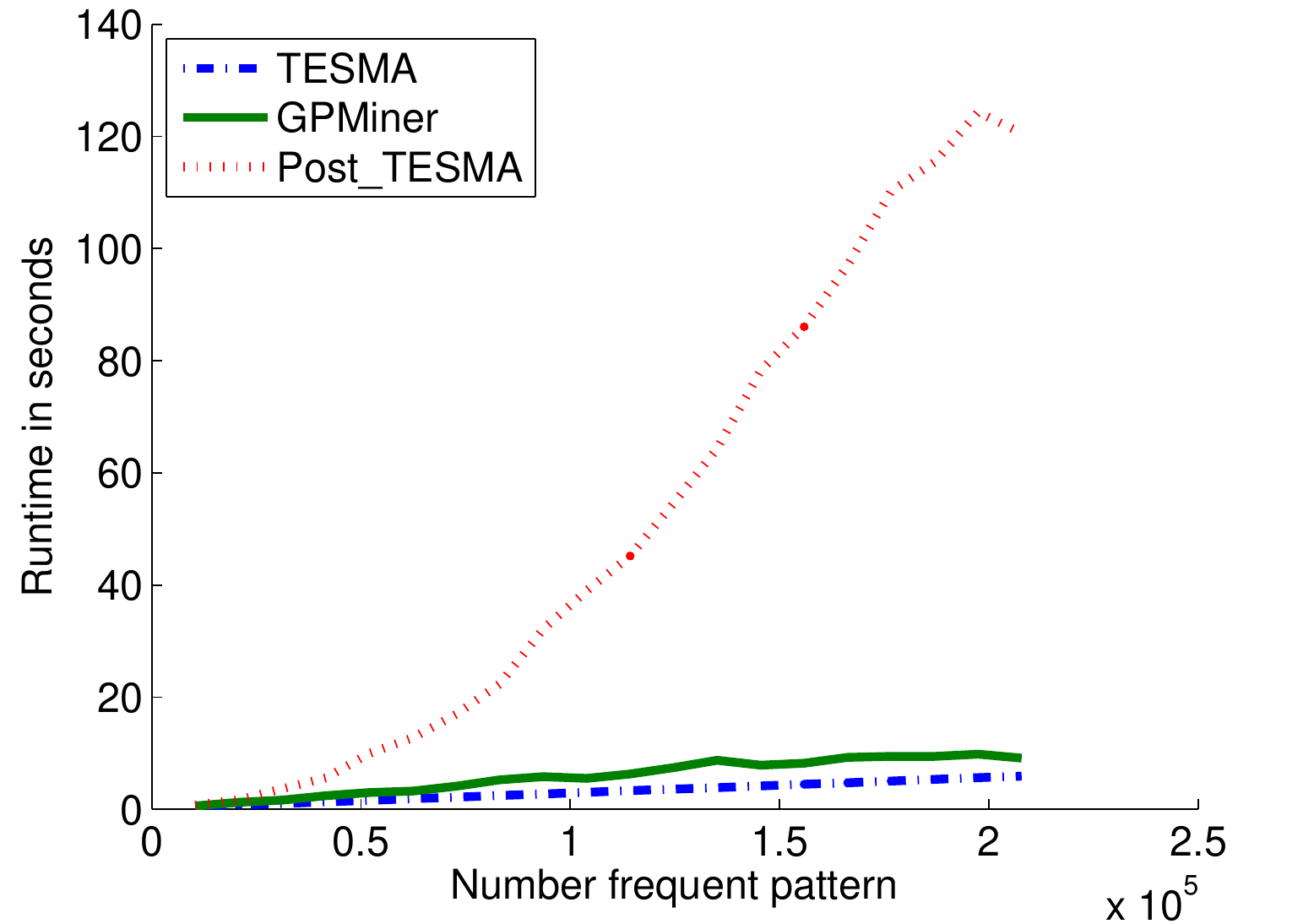}
		}
		\subfloat{
		\includegraphics[width=0.44\textwidth]{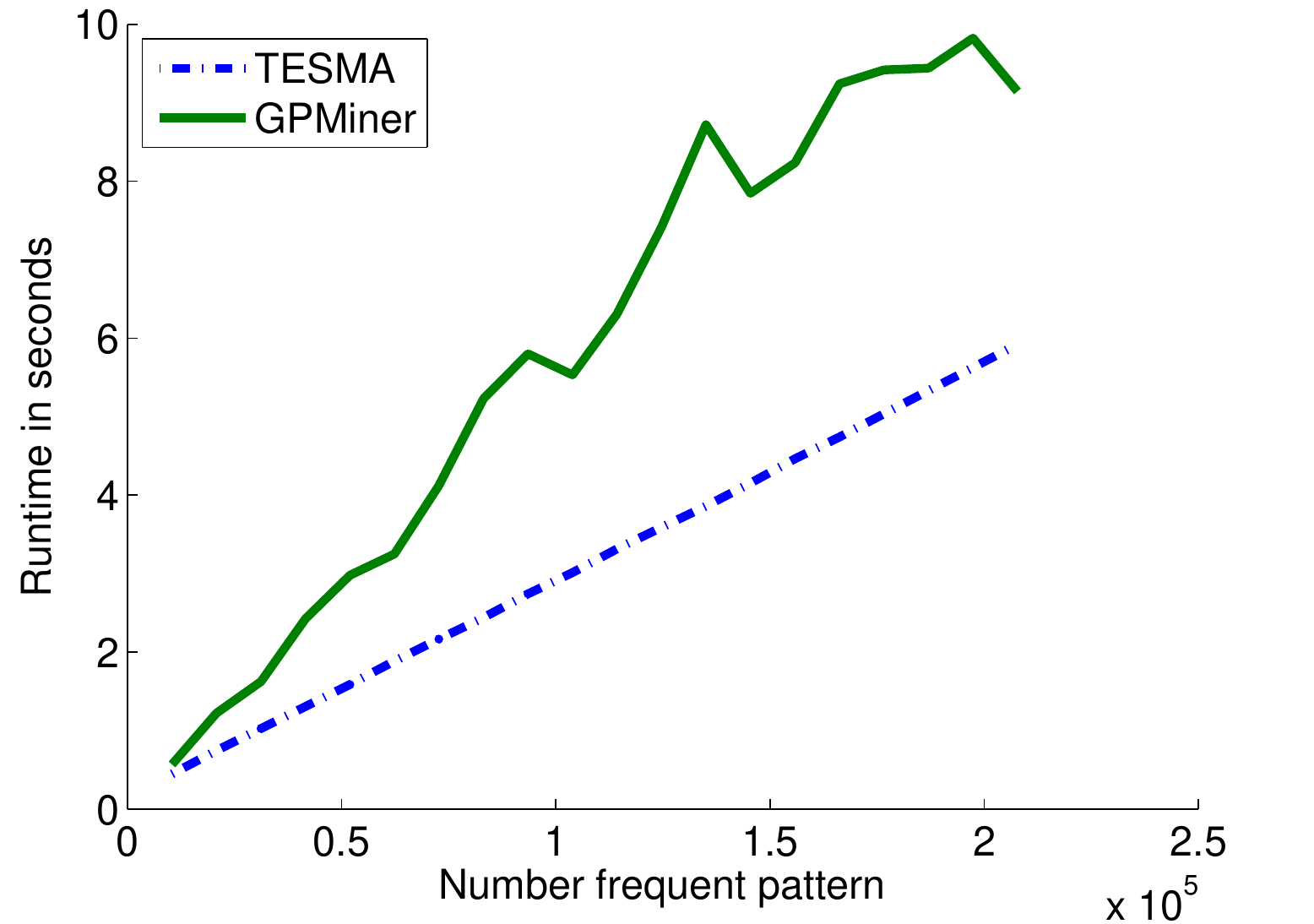}
		}\qquad
			\subfloat{
		\includegraphics[width=0.44\textwidth]{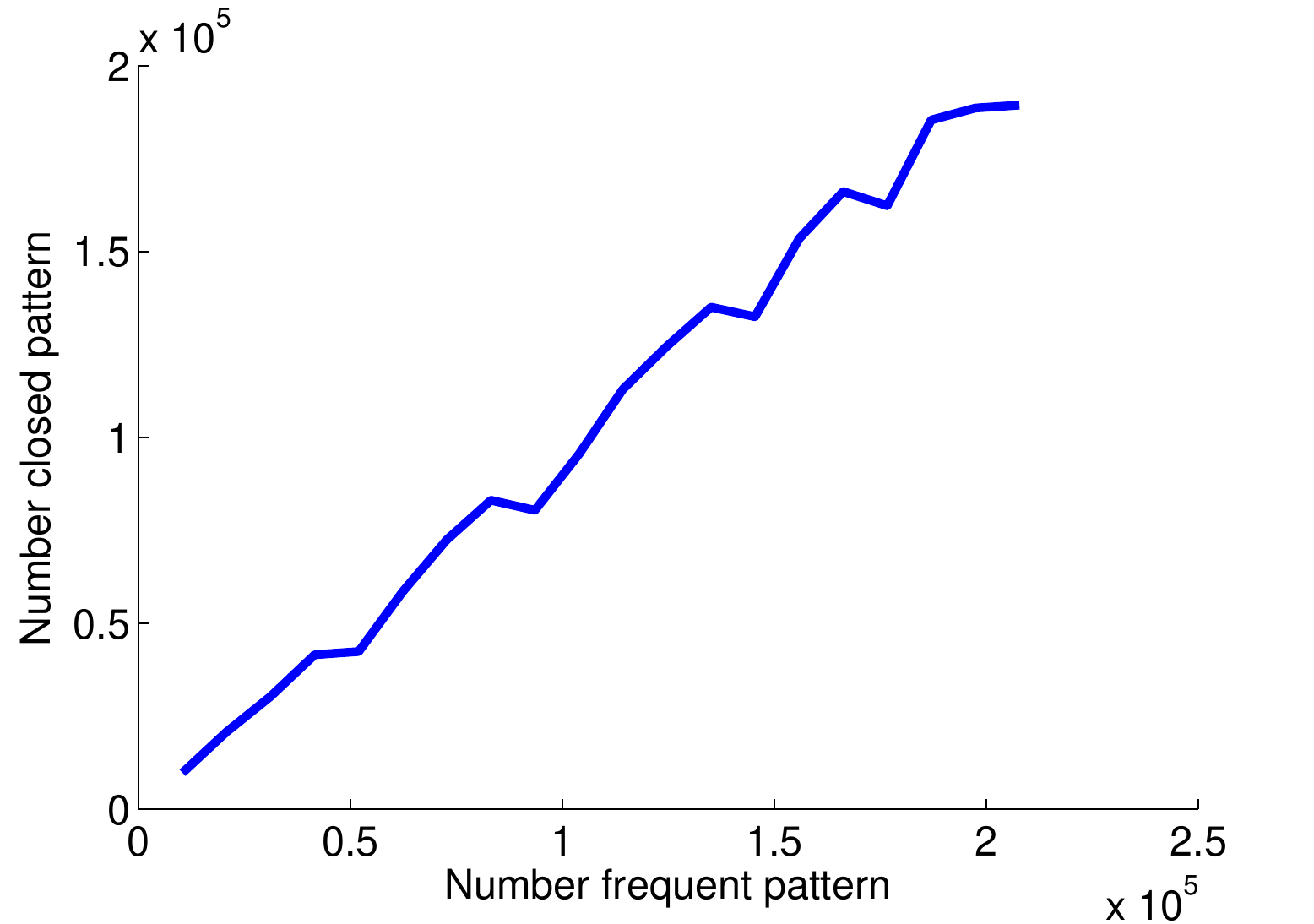}
	}
		\subfloat{
		\includegraphics[width=0.44\textwidth]{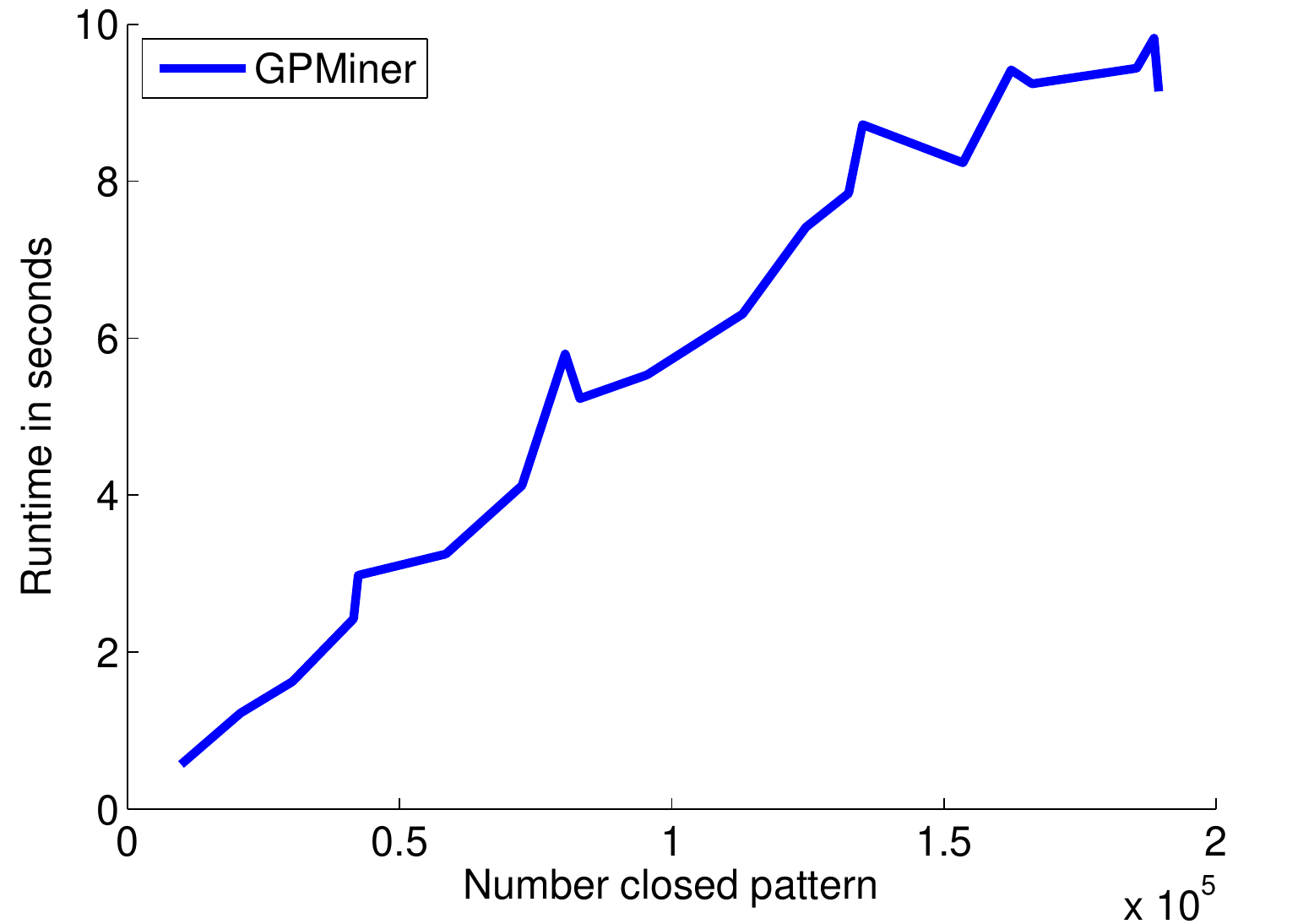}
		}
		\caption{Top: Runtime depending of the number of frequent patterns (on the right without Post\_TESMA). Bottom: Number of closed rankings against the number of frequent rankings (left), and the dependency between the runtime of GPMiner and the number of closed rankings (right).}
		\label{fig:incFr}
\end{figure}

Figure \ref{fig:incFr} shows the results of the first experiment. Since there are almost as many closed patterns as there are frequent ones, it is hardly surprising that TESMA shows the best performance. Strikingly, the runtime and the number of frequent rankings show an almost perfect linear dependency for TESMA.
Nevertheless, GPMiner is still several magnitudes faster than Post\_TESMA. 
Figure \ref{fig:incFr}(d) plots the runtime of GPMiner against the number of closed rankings. Again, a linear dependency can be recognized, although the curve is not very smooth.

\subsection{Semi-synthetic data I}
\label{sec:synDataII}

Our second experiment is conducted on data that is semi-synthetic in the sense of being a modification of real-world data. It is based on the well-known SUSHI dataset, which contains $5000$ rankings of $10$ different types of sushi. 
In a first series of experiments, we artificially increased the size of this real-world dataset by a factor of $v \in \{2,4,\ldots,20\}$.
To this end, the dataset is first copied $v$ times. Then, in each ranking of each copy, the neighbored pairs of items are switched with a probability of $0.01$.

\begin{figure}
	\centering
		\subfloat{
		\includegraphics[width=0.32\textwidth]{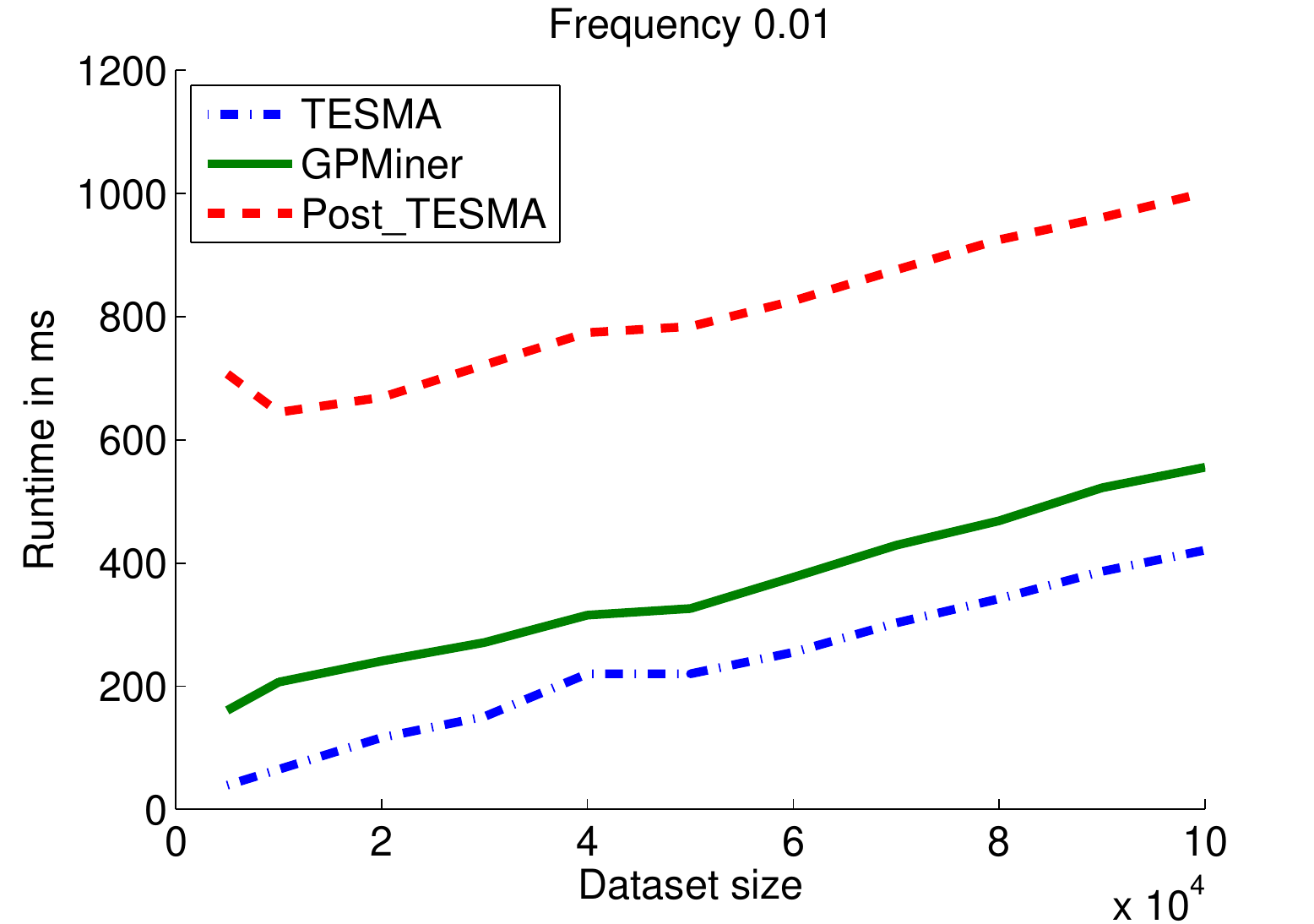}
		}
	\subfloat{
		\includegraphics[width=0.32\textwidth]{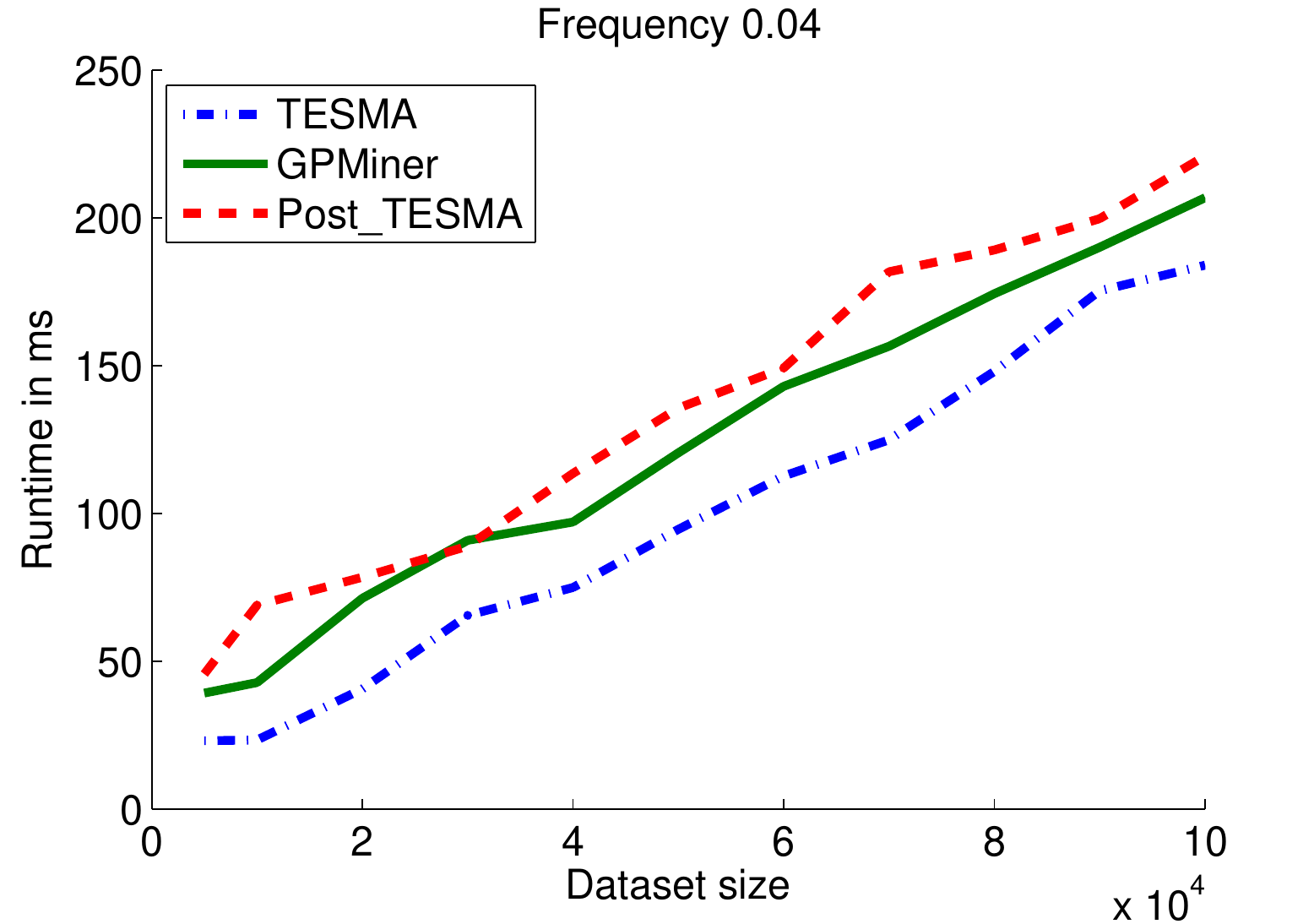}
	}
	\subfloat{
		\includegraphics[width=0.32\textwidth]{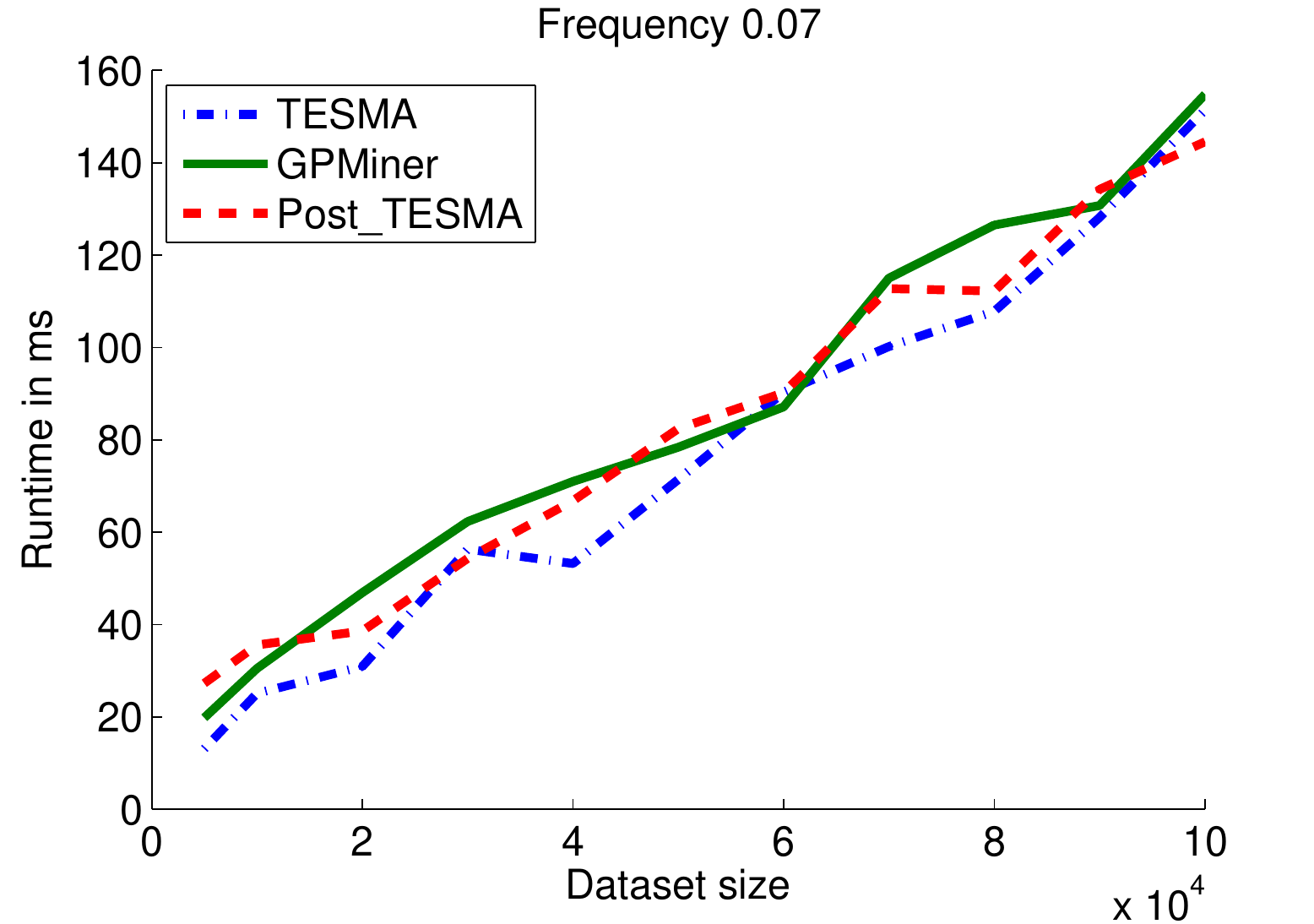}
	}
		\caption{Runtime depending on dataset size.}
		\label{fig:TvsGincN}
\end{figure}

The results are shown in Figure \ref{fig:TvsGincN}. As can be seen, the runtime again depends linearly on the dataset size for all three algorithms. Since the number of frequent (and hence closed frequent) rankings is decreasing with an increasing frequency threshold, the runtimes of all three algorithms are converging with increasing frequency.

\subsection{Semi-synthetic data II}
\label{sec:synDataIII}

In another series of experiments, we used the suhsi data of size $100000$ and increased the length of the rankings.
To this end, instead of inserting new objects purely at random (and hence producing a lot of noise), a 2-pattern $(i,j)$ is identified and extended in a systematic way. More concretely, if $\pi$ is a ranking of length $K$ and $\pi(j) = r$, the new object is randomly set on a position $r_{new}$ with $r < r_{new} \leq K+1$ (and all old objects with positions greater or equal $r_{new}$ are moved up by one position).
For instance, if the 2-pattern is $(a,d)$, then the possible extensions of the ranking $(a,c,d,e,b)$ are $(a,c,d,f,e,b)$, $(a,c,d,e,f,b)$, and $(a,c,d,e,b,f)$.
If a ranking does not contain the chosen 2-pattern, the new object is randomly set on a position $r_{new} \in \{ 1,  \ldots , K+1 \}$.

\begin{figure}
	\centering
		\subfloat{
		\includegraphics[width=0.32\textwidth]{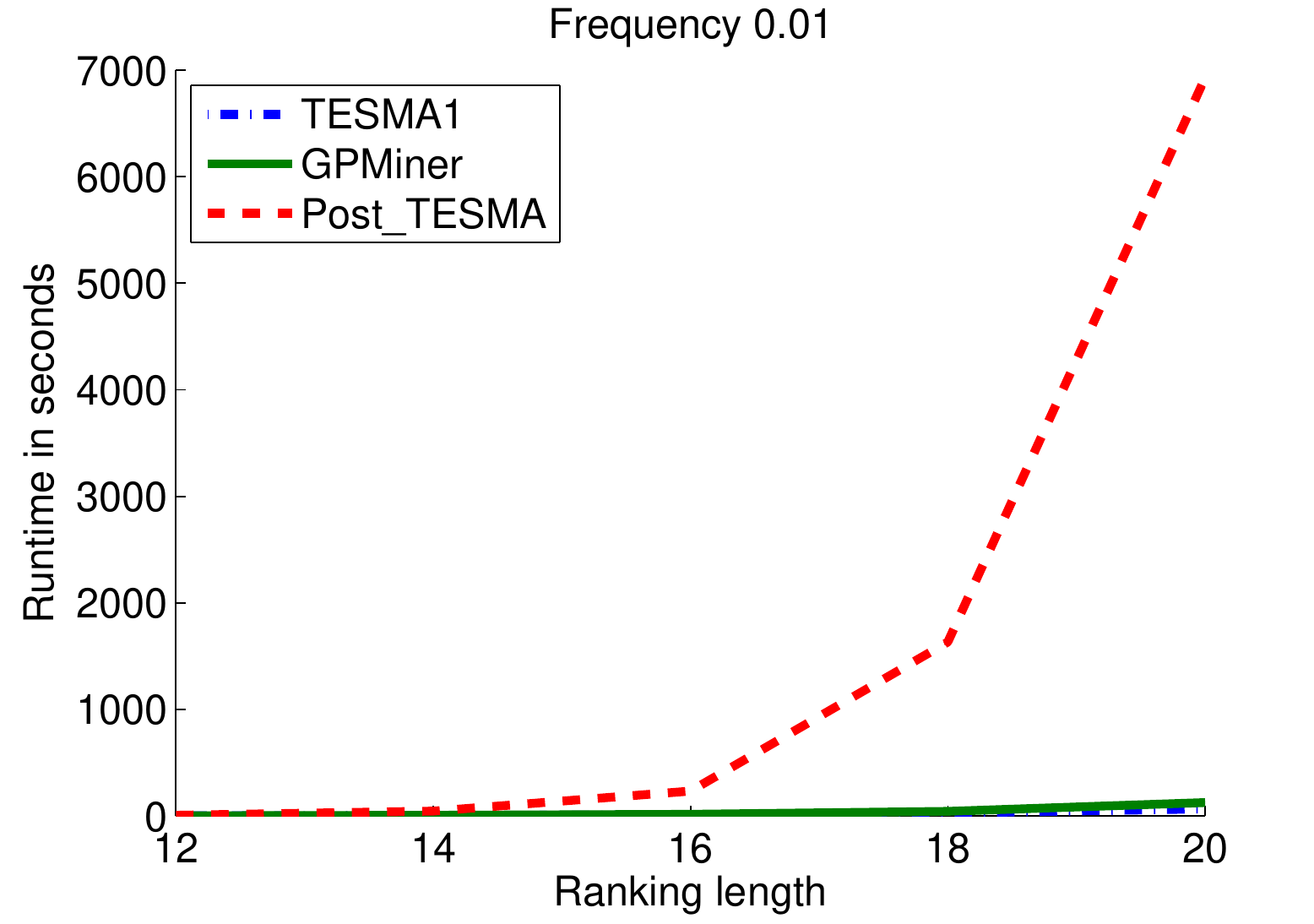}
		}
	\subfloat{
		\includegraphics[width=0.32\textwidth]{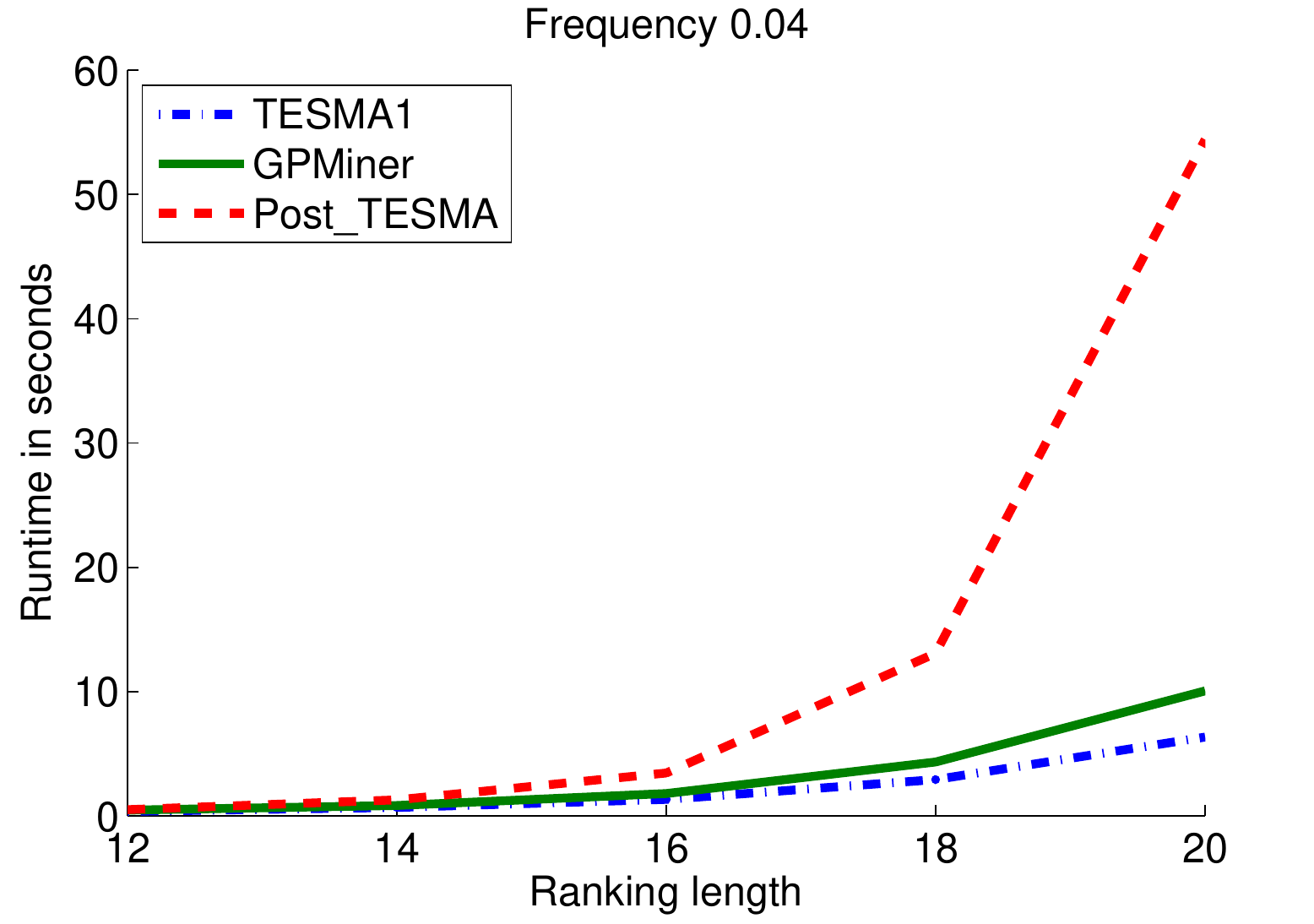}
	}
	\subfloat{
		\includegraphics[width=0.32\textwidth]{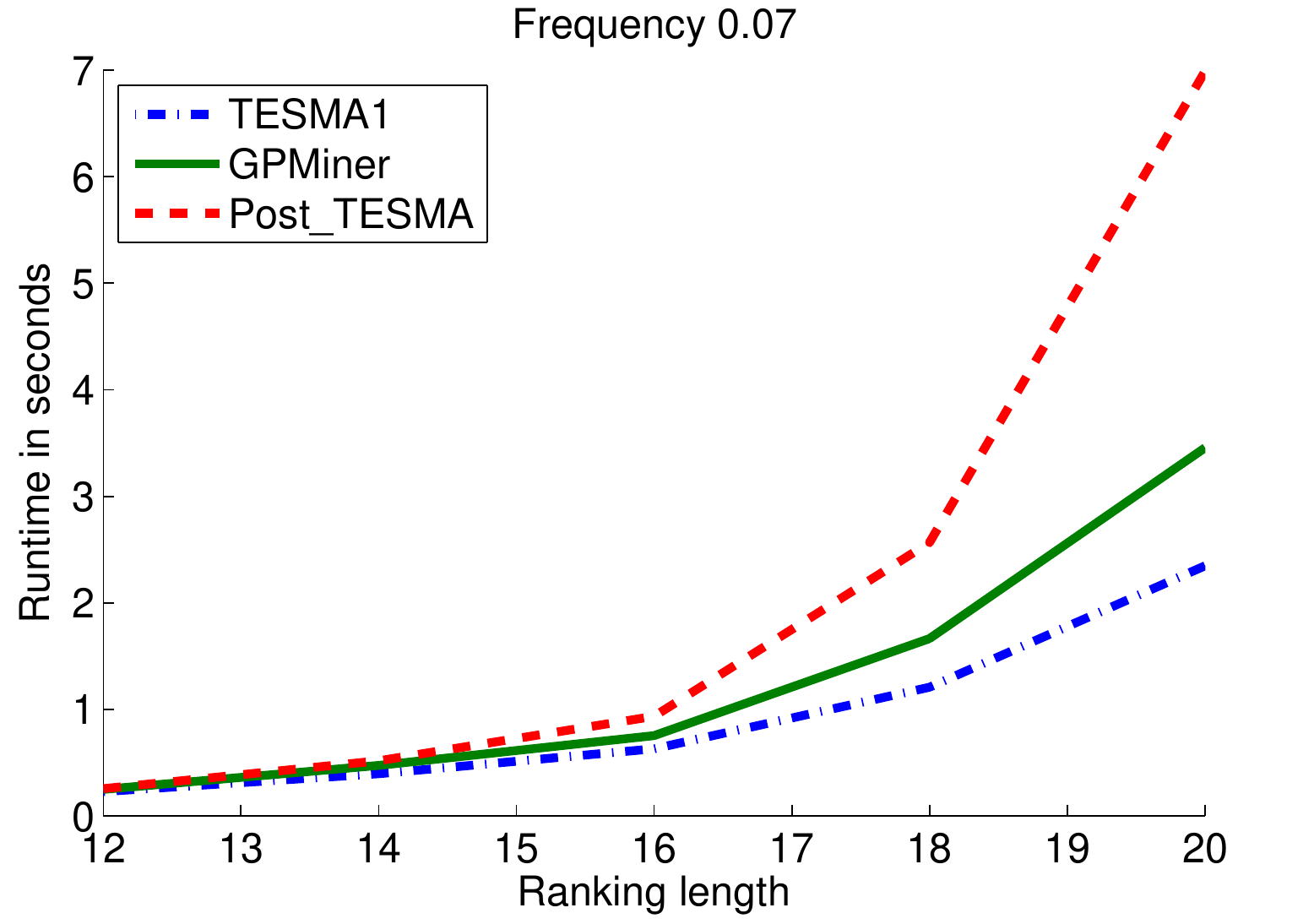}
	}
		\caption{Runtime dependency on dataset size.}
		\label{fig:TvsGincR}
\end{figure}

As can be seen in Figure \ref{fig:TvsGincR}, TESMA again performs best in this setting, suggesting a proportion between frequent and closed patterns that is in disfavor of GPMiner. Yet, GPMiner is again significantly faster than Post\_TESMA. Thus, when being interested in closed patterns, GPMiner would be the best choice.

\subsection{Comparison with sequence mining}

In Section \ref{sec:IandSvsR}, we already explained that mining rank data shares several commonalities with the problem of sequence mining. We also pointed out, however, that sequence mining is more general, and speculated that our methods are advantageous due to exploiting specific properties of rankings. The goal of this section is to validate this conjecture experimentally. 
To this end, we compared TESMA with the common sequence mining algorithms discussed in Section \ref{sec:IandSvsR}, using the implementation of the SPMF Open-Source Data Mining Library \cite{spmf}; in addition, we included the recent algorithms CM-SPAM and CM-SPADE, which have been shown to be very fast \cite{FournierViger2014}.


A first series of experiments was done using the semi-synthetic data from Sections \ref{sec:synDataII} and \ref{sec:synDataIII}. The frequency threshold was set to $0.01$. The results of the first experiment, in which the runtime\footnote{GSP was excluded from this experiment, as it simply too slow.} and memory consumption was analyzed when increasing the size of the data from $N=5000$ to $N=100000$, are shown in Figures \ref{fig:TESMAvsSPMFmem}(a) and and \ref{fig:TESMAvsSPMFrunN}. The results of a second experiment, in which the length of the rankings ($K$) was increased, 
can be seen in Figures \ref{fig:TESMAvsSPMFmem}(b) and \ref{fig:TESMAvsSPMFrunr}.

\begin{figure}
	\centering
		\subfloat[]{
		\includegraphics[width=0.45\textwidth]{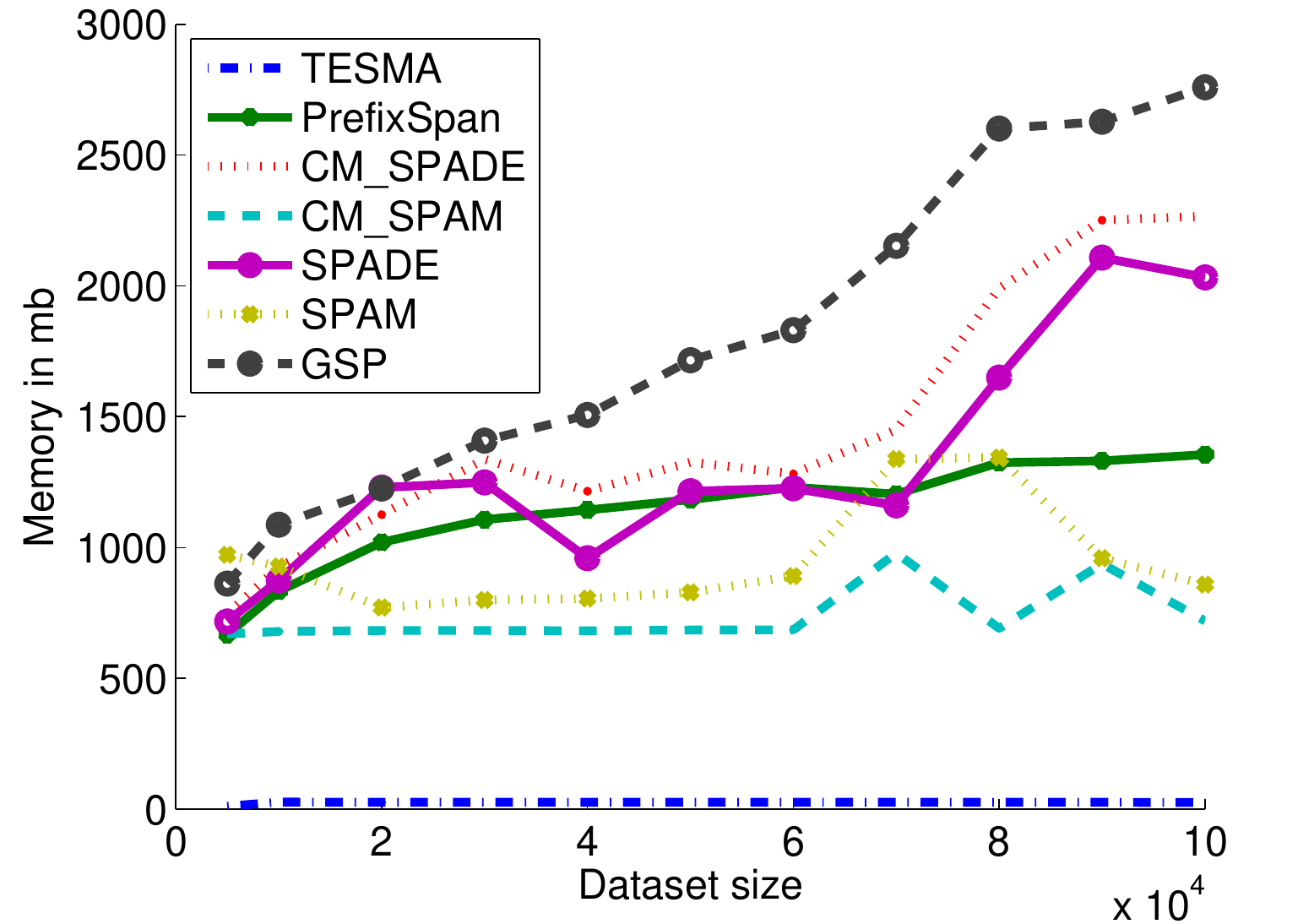}
		}
	\subfloat[]{
		\includegraphics[width=0.45\textwidth]{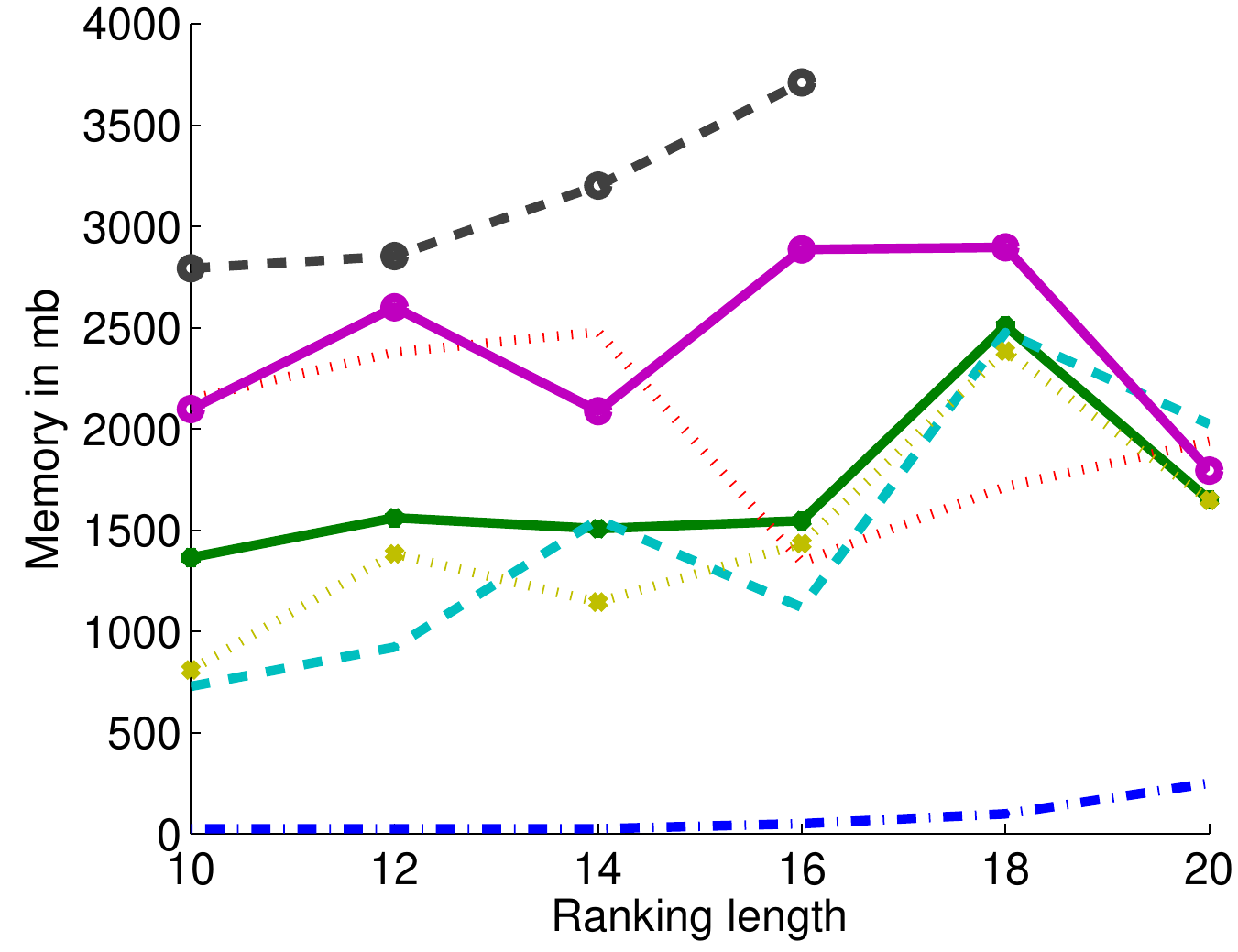}
	}
		\caption{Memory consumption of TESMA and several sequence miners as a function of the dataset size (a) and length of rankings (b).}
		\label{fig:TESMAvsSPMFmem}
\end{figure}

\begin{figure}
	\centering
		\subfloat[]{
		\includegraphics[width=0.45\textwidth]{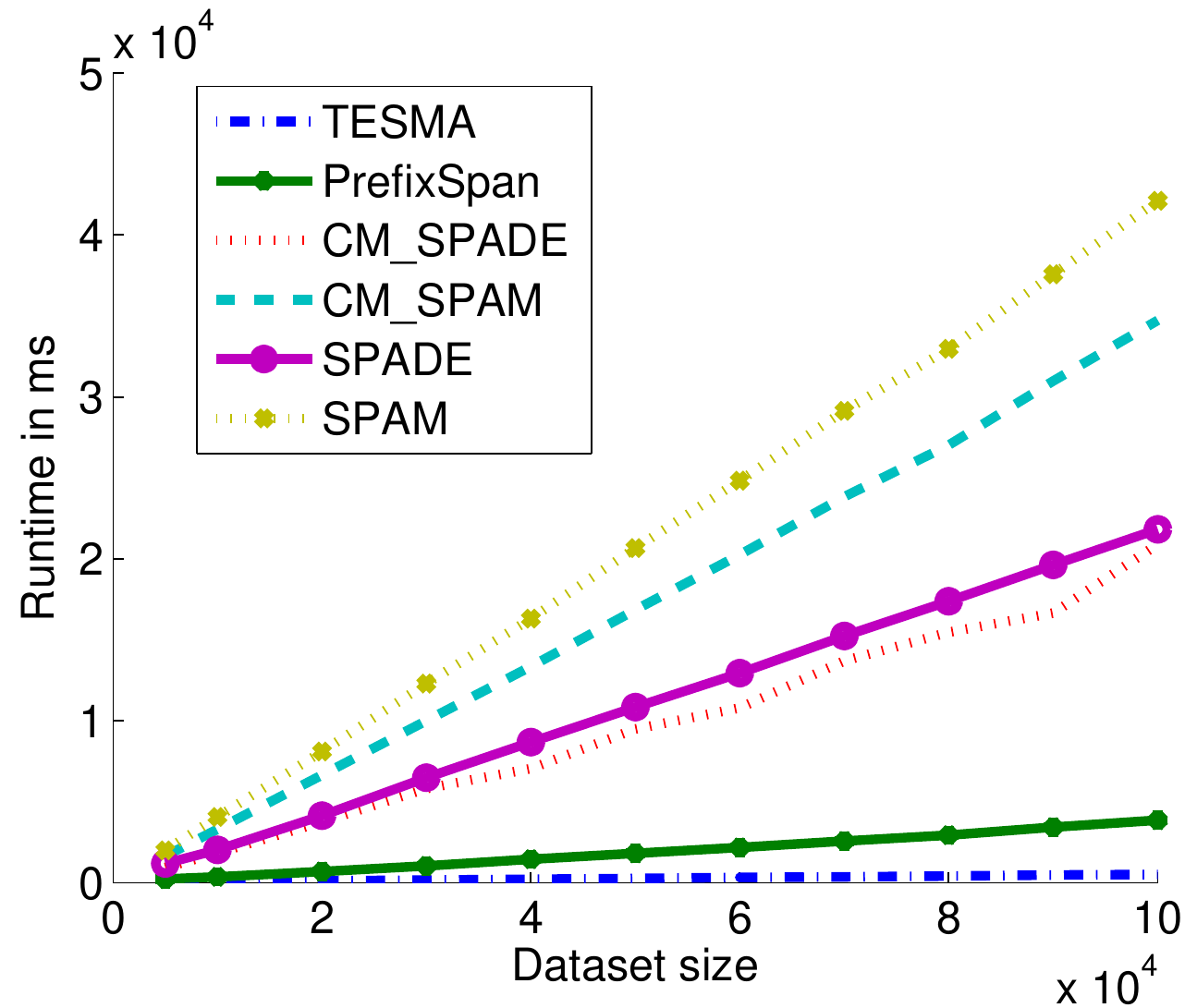}
		}
	\subfloat[]{
		\includegraphics[width=0.45\textwidth]{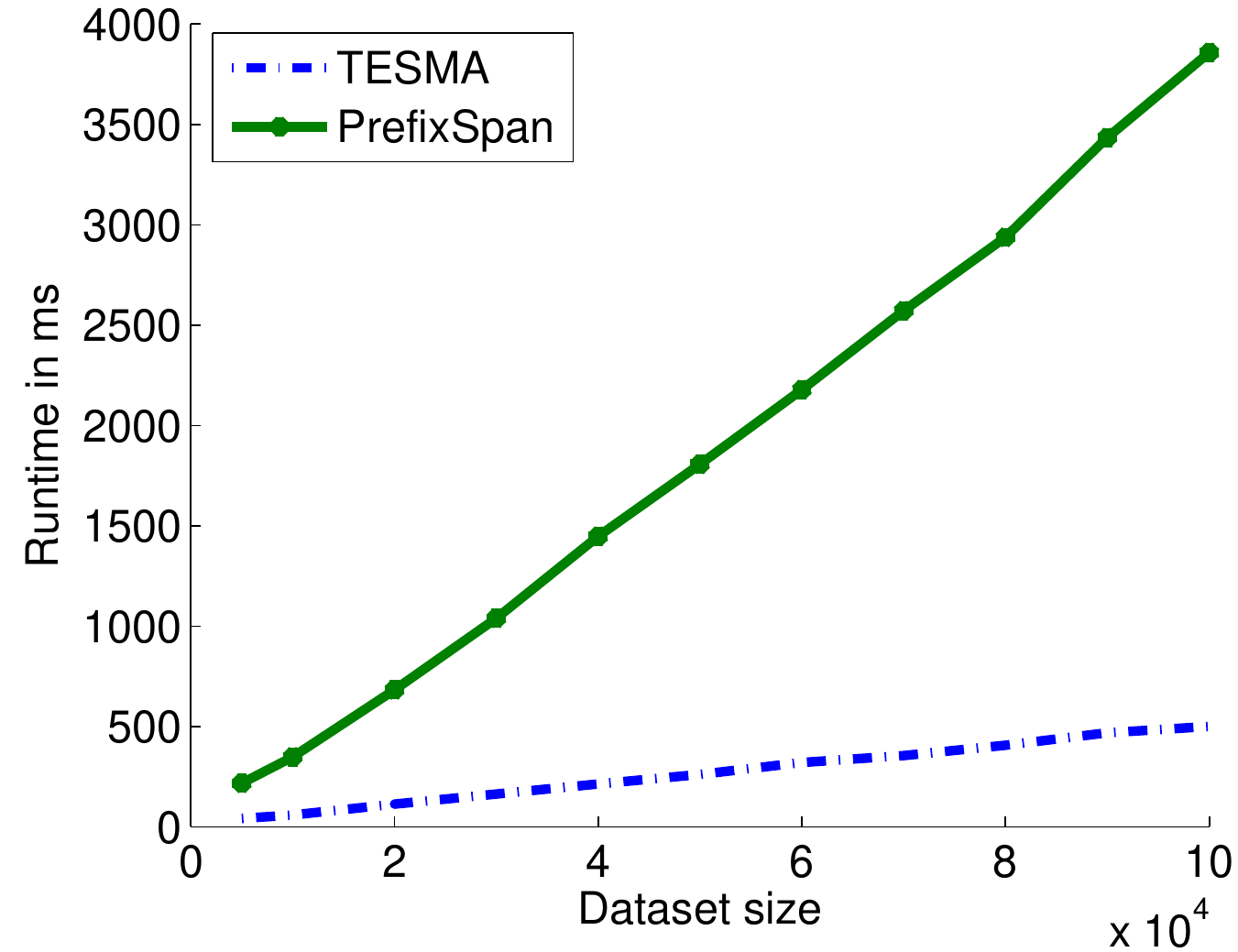}
	}
		\caption{Runtime of TESMA and several sequence miners as a function of dataset size. On the right, TESMA is compared to the fastest sequence miner found during the experiments.}
		\label{fig:TESMAvsSPMFrunN}
\end{figure}

\begin{figure}
	\centering
		\subfloat[]{
		\includegraphics[width=0.45\textwidth]{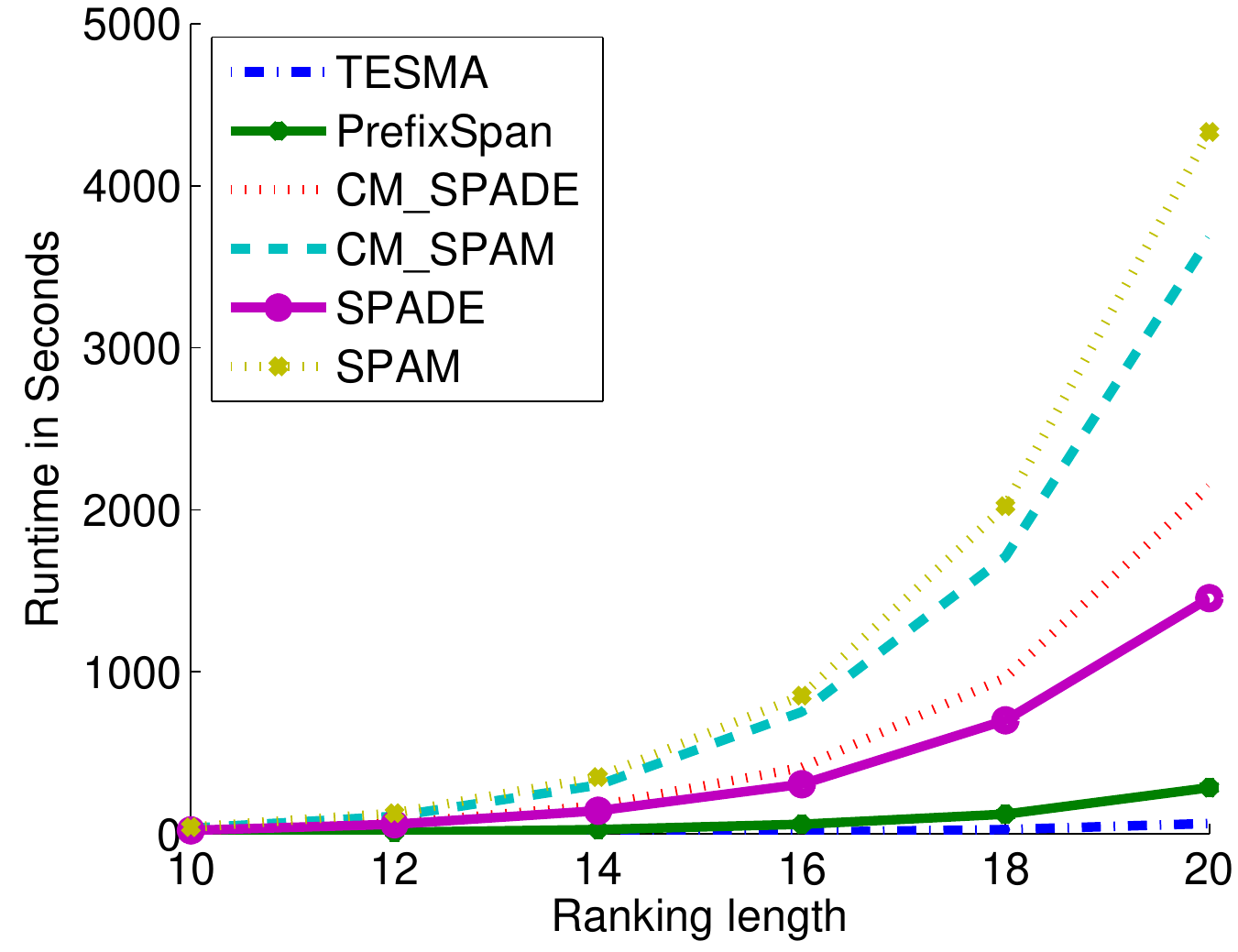}
		}
	\subfloat[]{
		\includegraphics[width=0.45\textwidth]{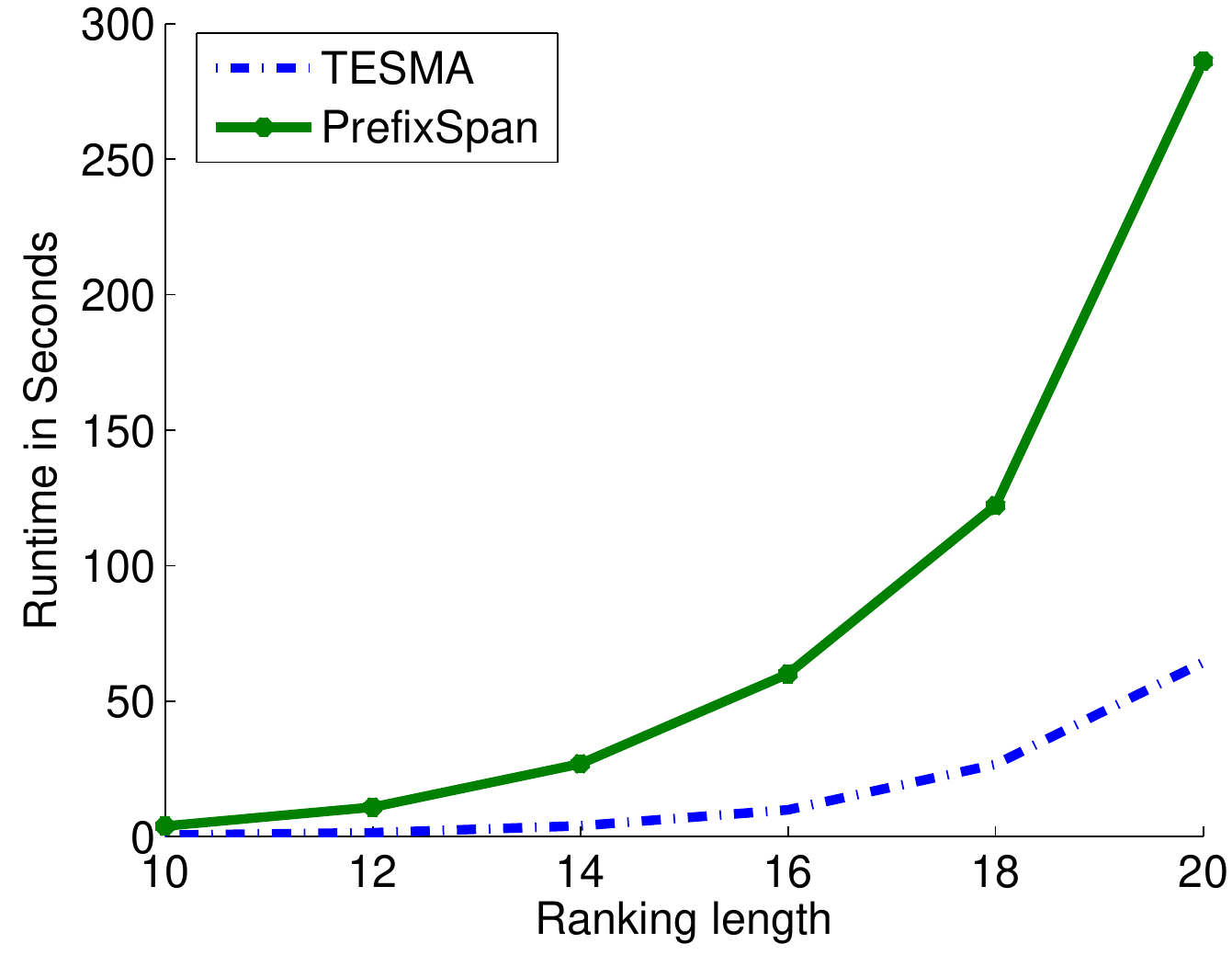}
	}
		\caption{Runtime of TESMA and several sequence miners as a function of the ranking length. On the right, TESMA is compared to the fastest sequence miner found during the experiments.}
		\label{fig:TESMAvsSPMFrunr}
\end{figure}


TESMA is clearly outperforming the other algorithms, both in terms of runtime and memory consumption. 
As expected, the runtime grows linearly with $N$ for all algorithms. The dependence of runtime on $K$ is more difficult to estimate, but definitely much worse---this is unsurprising in light of (\ref{eq:tnor}), showing that the number of potential patterns grows extremely fast with the number of items. 

The memory consumption of TESMA appears to be linear in $N$ and quadratic in $K$. 
For the other algorithms, the picture is less clear, also because the curves are sometimes quite erratic. This might be caused by implementation issues with the SPMF algorithms. For example, if too many objects are produced, the reference of which is lost during execution, a non-deterministic call of the garbage collector may strongly influence the results.


%
%


Similar experiments were conducted to compare GPMiner with closed sequence mining algorithms (see Section \ref{sec:IandSvsR}), again using the SPMF Open-Source Data Mining Library \cite{spmf}. In addition, we included a quite recent and efficient algorithm called CM-ClaSP \cite{FournierViger2014}.
Again, we ran experiments on our semi-synthetic data to compare runtime and memory consumption. This time, however, the dataset size was limited to $N=40000$ (instead of 100000), because many sequence miners ran out of memory for larger $N$. For the same reason, rankings of different length $K$ were produced for a basic dataset of size $N=5000$ (instead of $100000$).

\begin{figure}
	\centering
	\subfloat[]{
		\includegraphics[width=0.45\textwidth]{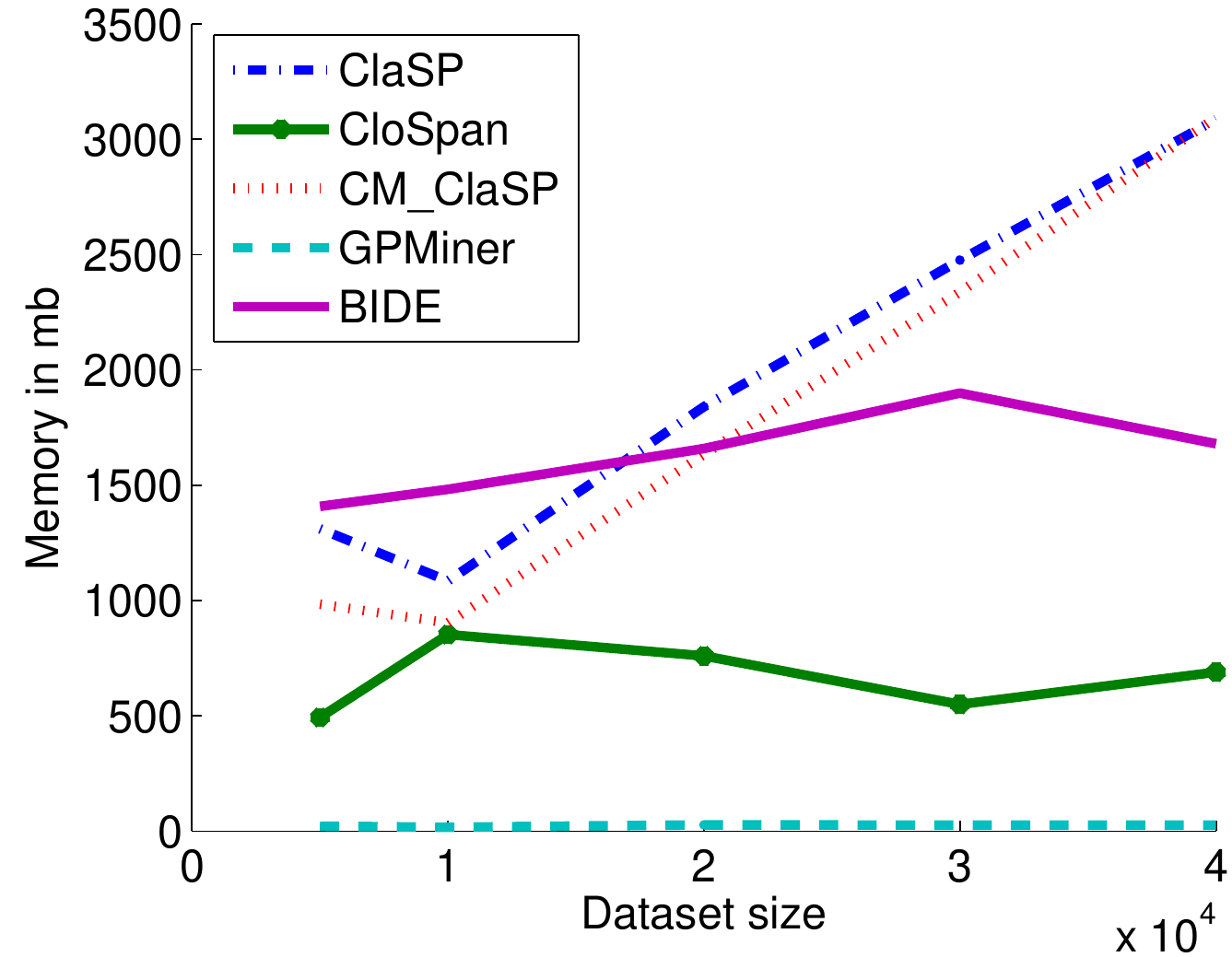}
		}
		\subfloat[]{
		\includegraphics[width=0.45\textwidth]{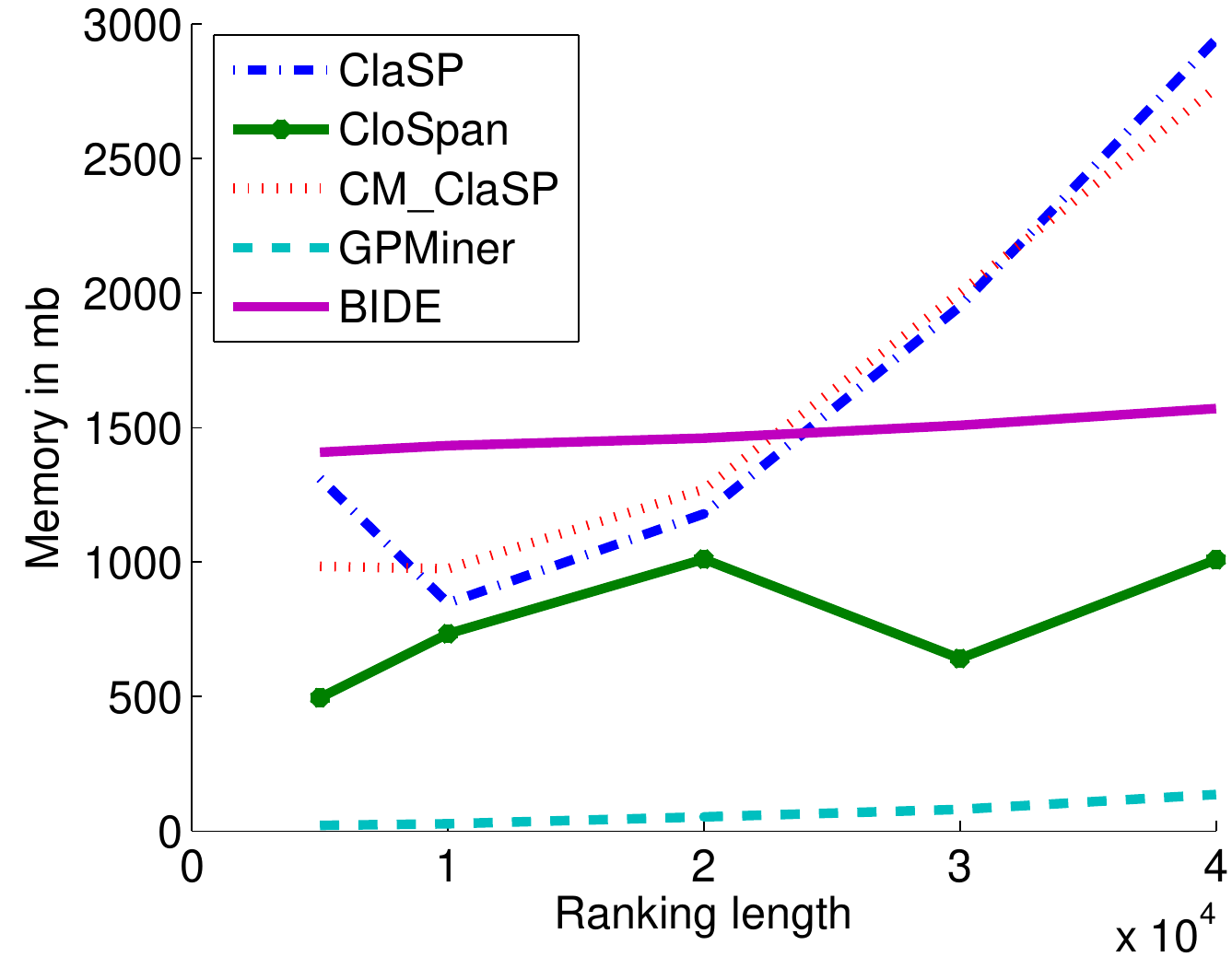}
		}
		\caption{Memory consumption of sequence miners and GPMiner as a function of dataset size (left) and ranking length (right).}
		\label{fig:GPMinerVsSPMFmemIncNr}
\end{figure}

Figure \ref{fig:GPMinerVsSPMFmemIncNr} depicts the memory consumption as a function of $N$. As in the previous experiment, the curves for the (closed) frequent sequence miners do not show any clear trend, probably again due to memoryleaks in the SPMF implementations. In terms of runtime, GPMiner is outperforming all closed sequence miners (Figure \ref{fig:GPMinerVsSPMFrunIncNr}). The algorithm BIDE is excluded from the plot, because it performed orders of magnitude worse than all other algorithms.

\begin{figure}
	\centering
	\subfloat[]{
		\includegraphics[width=0.45\textwidth]{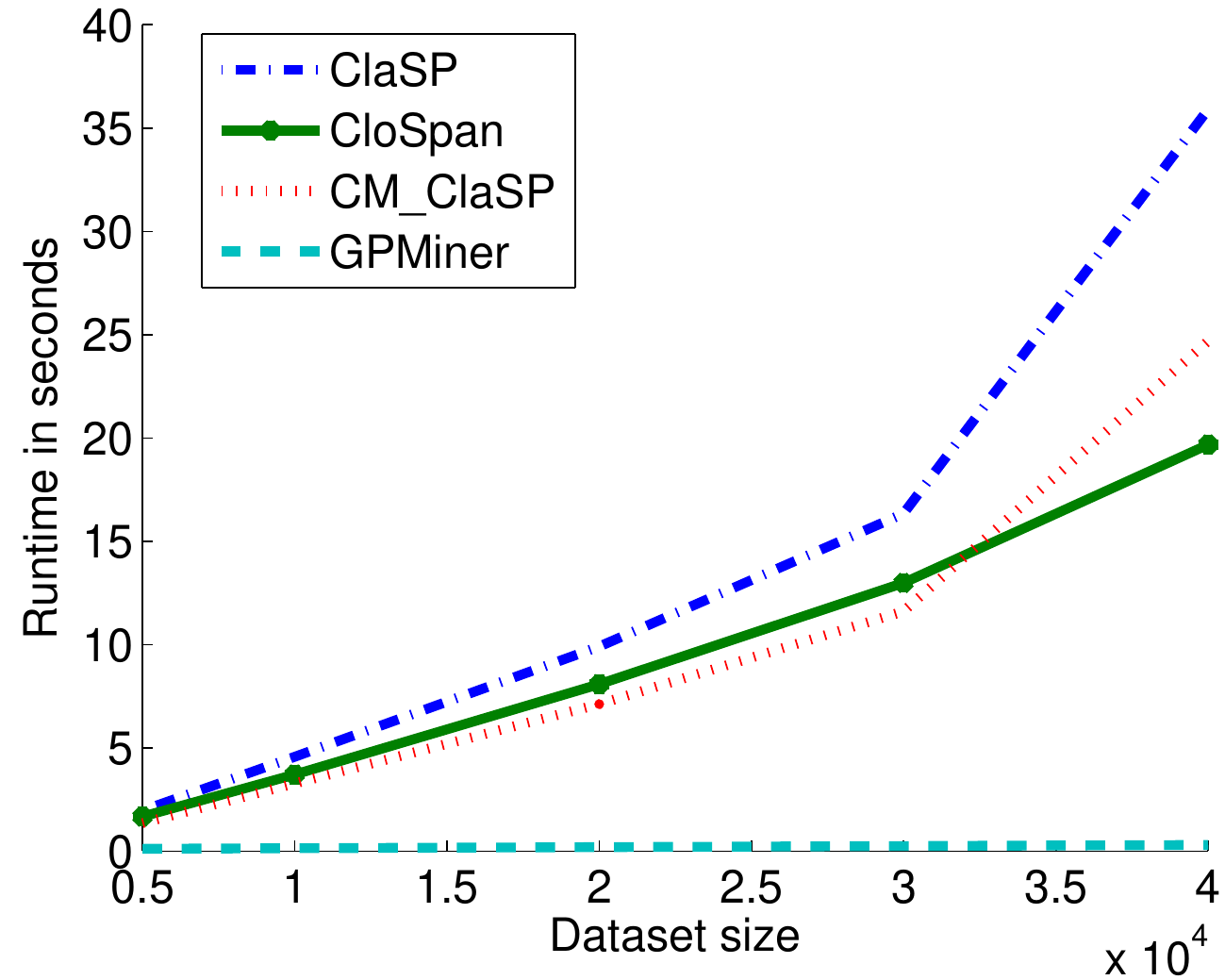}
		}
		\subfloat[]{
		\includegraphics[width=0.45\textwidth]{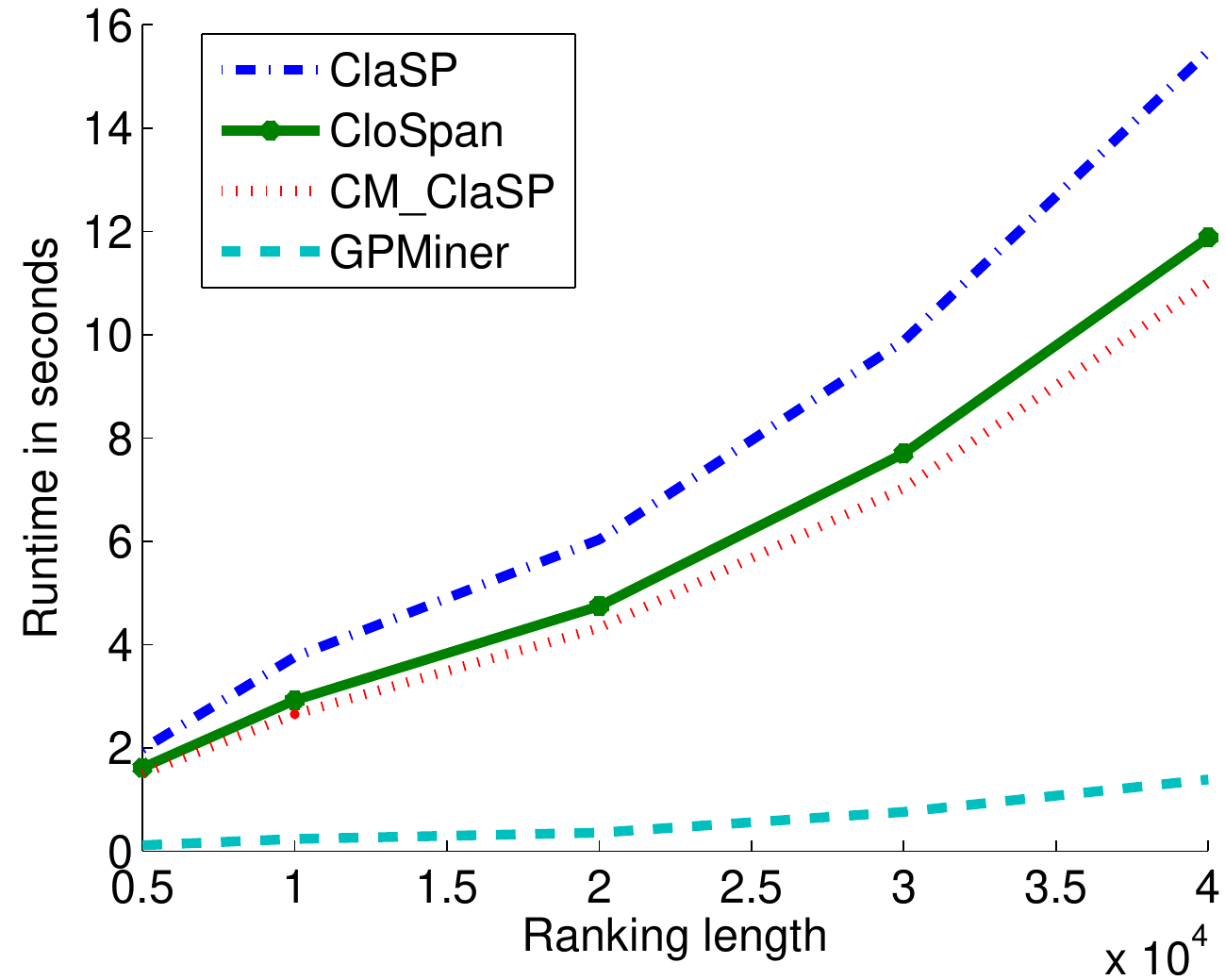}
		}
		\caption{Runtime of sequence miners and GPMiner as a function of dataset size (left) and ranking length (right).}
		\label{fig:GPMinerVsSPMFrunIncNr}
\end{figure}


%
%

\subsection{Association rules for the SUSHI data}

Our last experiment is meant to illustrate the notion of association rules in the context of rank data. To this end, we produced a number of strong associations for the SUSHI dataset that was already mentioned in the introduction. More specifically, association rules were mined with a support threshold of $0.1$ and $0.2$, respectively, and the rules were sorted in decreasing order of their interest (difference between confidence of the rule and confidence of the default rule, i.e., support of the rule consequent).

As can be seen from the top-rules summarized in Table \ref{tab:sushiInterest}, there are indeed several strong associations between preferences, with a confidence larger than 0.9 or an interest of up to $0.59$. To some extent, this is remarkable, because the SUSHI data is known to be highly irregular and in a sense very noisy; for example, predicting complete rankings of individual customers is extremely difficult, and truly accurate predictions are apparently not possible \cite{mpub262}. Preference patterns discovered by our mining algorithms could be useful to make predictions that are at least locally valid. While interesting, elaborating on applications of that kind is certainly beyond the scope of this paper.

It might be noticeable that, in all rules shown in Table \ref{tab:sushiInterest}, the antecedent and consequent part are sharing at least one item. As a possible explanation, note that the more items are shared, the stronger the dependence tends to be (if a transaction covers an antecedent, it is likely to cover the consequent, too), and hence the higher the measure of interest.

\begin{table}
	\centering
		\caption{Best association rules on the SUSHI data in terms of interest (the first 5 rules for a support threshold of $0.1$, the second 5 for a threshold of $0.2$).}
	\label{tab:sushiInterest}
		\begin{tabular}{|c|c|c|}\hline
		rule 										& confidence 	& interest \\\hline
$(8,4,10,5) \entails (8,7,5)$ & (0.871) & (0.597)\\
$(8,3,10,5) \entails (8,9,5)$ & (0.929) & (0.594)\\
$(8,1,10,5) \entails (8,7,5)$ & (0.863) & (0.59)\\
$(8,3,7,6) \entails (8,9,6)$ & (0.878) & (0.589)\\
$(3,1,10,5) \entails (3,7,5)$ & (0.84) & (0.582)\\\hline\hline

$(3,10,5) \entails (9,5)$ & (0.982) & (0.511)\\
$(1,10,5) \entails (7,5)$ & (0.922) & (0.51)\\
$(3,10,5) \entails (7,5)$ & (0.921) & (0.509)\\
$(9,10,5) \entails (7,5)$ & (0.918) & (0.507)\\
$(5,8,9) \entails (5,3)$ & (0.942) & (0.504)\\\hline
		\end{tabular}
\end{table}

\section{Conclusion}
\label{sec:conclusion}

In this paper, we give a motivation for mining rank data and provide a formalization of this problem. Moreover, we position the problem in the context of related work on frequent pattern mining, such as itemset and sequence mining.   

The main contributions of the paper are two algorithms for mining rank data, called TESMA and GPMiner, respectively. While TESMA finds frequent (sub-)rankings, using the standard concept of support, GPMiner seeks to discover rankings that are not only frequent but also closed. Finding closed rankings turns out to be more difficult than finding frequent ones, both theoretically and practically, especially because the commonly used Galois operator can not be applied in its standard form any more. 

To show the efficiency of our algorithms and to get an idea of how they perform in practice, we conducted a number of experimental studies. The use of appropriate synthetic and semi-synthetic data allowed us to control certain properties of the data, such as the ratio between closed and frequent patterns, which is important for comparing TESMA and GPMiner. We also compared our algorithms with existing algorithms for sequence mining, which solve a problem that is in a sense more general than ours. Consequently, they cannot compete with our algorithms, which exploit specific properties of rank data; thanks to this, TESMA and GPMiner achieve substantial speedups in comparison to sequence miners. 

Needless to say, the current paper is only a first step toward a complete methodology for mining rank data. Correspondingly, it can be extended in various directions. What we are concretely planning as next steps is to examine algorithms for mining maximal rankings and data containing rankings with ties. Besides, as already suggested at the end of Section 6, the use of preference mining for supporting (supervised) preference learning is another interesting direction to be explored in future work. 


\end{document}